\documentclass[journal]{IEEEtran}
\usepackage{amsmath}
\usepackage{amsthm}
\usepackage{amsfonts}
\usepackage{dsfont}
\usepackage{enumitem}
\usepackage{graphicx}
\usepackage{multirow}
\usepackage{subfigure}
\usepackage{algorithm}
\usepackage{algorithmic}
\usepackage{color}
\usepackage{cite}
\usepackage{array}
\newcommand{\PreserveBackslash}[1]{\let\temp=\\#1\let\\=\temp}
\newcolumntype{C}[1]{>{\PreserveBackslash\centering}p{#1}}
\newcolumntype{R}[1]{>{\PreserveBackslash\raggedleft}p{#1}}
\newcolumntype{L}[1]{>{\PreserveBackslash\raggedright}p{#1}}

\newtheorem{myDef}{Definition}
\newtheorem{myTheo}{Theorem}

\ifCLASSINFOpdf
\else
\fi

\begin{document}
%
\title{Robust Subspace Clustering by \\Cauchy Loss Function}
%
%
%

\author{Xuelong Li,~\IEEEmembership{Fellow,~IEEE}, Quanmao Lu, Yongsheng Dong,~\IEEEmembership{Member,~IEEE}, and Dacheng Tao,~\IEEEmembership{Fellow,~IEEE} 
\thanks{This work was supported in part by The National Key Research and Development Program of China under Grant 2018YFB1107400,
in part by the National Natural Science Foundation of China under Grants 61871470, 61761130079, U1604153,and 61301230,
in part by the Program for Science and Technology Innovation Talents in Universities of Henan Province under Grant 19HASTIT026,
and in part by the Training Program for the Young-Backbone Teachers in Universities of Henan Province under Grant 2017GGJS065. (Corresponding author: Yongsheng Dong.)

X. Li, Q. Lu, and Y. Dong are with the Xi'an Institute of Optics and Precision Mechanics, Chinese Academy of Sciences, Xi'an 710119, Shaanxi, P. R. China (emails: xuelong\_li@opt.ac.cn, quanmao.lu.opt@gmail.com,
dongyongsheng98@163.com). Y. Dong is also with the School of Information Engineering, Henan University of Science and Technology, Luoyang 471023, Henan, P. R. China.

D. Tao is with the UBTech Sydney Artificial Intelligence Institute and
the School of Information Technologies in the Faculty of Engineering and
Information Technologies, The University of Sydney, Darlington NSW 2008,
Australia (e-mail: dacheng.tao@gmail.com).}
}

%
%

\markboth{Transactions on Neural Networks and Learning Systems}%
{Shell \MakeLowercase{\textit{et al.}}: Bare Demo of IEEEtran.cls for Journals}
\maketitle

\begin{abstract}

Subspace clustering is a problem of exploring the low-dimensional subspaces of high-dimensional data. State-of-the-arts approaches are designed by following the model of spectral clustering based method. These methods pay much attention to learn the representation matrix to construct a suitable similarity matrix and overlook the influence of the noise term on subspace clustering. However, the real data are always contaminated by the noise and the noise usually has a complicated statistical distribution. To alleviate this problem, we in this paper propose a subspace clustering method based on Cauchy loss function (CLF). Particularly, it uses CLF to penalize the noise term for suppressing the large noise mixed in the real data. This is due to that the CLF's influence function has a upper bound which can alleviate the influence of a single sample, especially the sample with a large noise, on estimating the residuals. Furthermore, we theoretically prove the grouping effect of our proposed method, which means that highly correlated data can be grouped together. Finally, experimental results on five real datasets reveal that our proposed method outperforms  several representative clustering methods.
\end{abstract}


%
\IEEEpeerreviewmaketitle

\section{Introduction}
\IEEEPARstart{S}ubspace clustering, as an important clustering analysis technique, has gained much attention in recent years and has numerous applications in image processing and computer vision, \emph{e.g.} image representation \cite{hong2006multiscale}, motion segmentation \cite{yan2006general}, saliency detection \cite{lang2012saliency} and image clustering \cite{ho2003clustering, Cui2018Subspace}. It aims to explore the low dimensional structure lying in the high-dimensional data. Particularly, conventional PCA \cite{Smith2002A} can be regarded as a special subspace clustering method which finds a single low-dimensional subspace of the high-dimensional data. However, in practice, data are always drawn from multiple low-dimensional subspaces and each subspace has different dimension. For example, the trajectories of different motion objects usually belong to different affine subspaces, or face images of individuals under varying pose may lie in different linear subspaces. Motivated by these, subspace clustering is designed for seeking the low-dimensional subspace of the raw data and clustering the data into groups with each group fitting a subspace. Furthermore, subspace clustering problem is formally defined as follow:
\begin{myDef}
(Subspace clustering) Given a set of sufficiently sampled data vectors ${\bf{X}} = [{{\bf{X}}_1},...,{{\bf{X}}_k}] = [{{\bf{x}}_1},...,{{\bf{x}}_n}] \in {\mathds{R}^{d \times n}}$, where $d$ represents the feature dimension and $n$ is the number of data. Assume that the data are drawn from a union of $k$ subspaces $\{ {S_i}\} _{i = 1}^k$, and $X_i$ be a collection of $n_i$ points drawn from the subspace $S_i$, $n = \sum\nolimits_{i = 1}^k {{n_i}}$. The task of subspace clustering is to segment the data according to the underlying subspaces they are drawn from.
\end{myDef}

In the past two decades, many advances have been done to improve the performance of subspace clustering \cite{7312458,tipping1999mixtures,7046379,bradley2000k,favaro2011closed,6803941,DBLP:conf/ccpr/LuLDT16}. They can be roughly divided into four categories, including algebraic methods \cite{vidal2005generalized, costeira1998multibody}, iterative methods \cite{zhang2009median,tseng2000nearest}, statistical methods \cite{ma2007segmentation,rao2008motion} and spectral clustering based methods \cite{you2016scalable,liu2010robust,elhamifar2009sparse,liu2011latent,7892917}. Most recently, spectral clustering based methods have shown its excellent performance in many applications.
In general, spectral clustering based methods are consisted of two main steps. Firstly, a similarity or affinity matrix is constructed to represent the similarity between the samples in the raw data. Secondly, a spectral clustering algorithm is employed to divide the raw data into $k$ groups based on the learned similarity matrix. Note that, how to build a proper similarity matrix plays a decisive role in the process of subspace clustering. So most spectral clustering based models were proposed to construct a more efficient similarity matrix.

Reviewing the existing methods, a similarity matrix is generally constructed using a self-expression model which regards the data itself as a dictionary to learn a representation matrix \cite{parsons2004subspace, Lu2018Structure}. Such a self-expression model assumes that the samples can be well represented using the points in the same subspace and the learned representation matrix can capture the similarity between the samples in the raw data. Ideally, the learned representation matrix should be block-diagonal \cite{lu2012robust,tang2016subspace}, which means the affinities of samples between cluster are all zeros. Considering the real data usually contain noise, a loss function is employed to deal with the noise.
Then the general model of spectral clustering based methods can be formulated as

\begin{equation}\label{1}
\begin{array}{*{20}{l}}
{\mathop {\min }\limits_{{\bf{Z}},{\bf{E}}} \varphi ({\bf{E}}) + \delta ({\bf{Z}})}\\
{s.t.{\kern 1pt} {\kern 1pt} {\kern 1pt} {\kern 1pt} {\bf{X}} = {\bf{XZ}} + {\bf{E}},}
\end{array}
\end{equation}
where $\bf{X}$ is the original data matrix, $\bf{Z}$ is the representation matrix and $\bf{E}$ represents the noise matrix. The functions of $\varphi ( \bf{E} )$ and $\delta ( \bf{Z} )$ are designed for restricting $\bf{E}$ and $\bf{Z}$ respectively. In many works, $\varphi ( \bf{E} )$ and $\delta ( \bf{Z} )$ are two properly norms. For example, Sparse Subspace Clustering (SSC) \cite{elhamifar2009sparse} uses $\ell_1$ norm to regularize the matrix $\bf{Z}$ for seeking the most sparsest representation of each point and chooses Frobenius norm to deal with the noise term $\bf{E}$. Different with SSC, Low-Rank Representation (LRR) \cite{liu2013robust} employs the nuclear norm to regularize the matrix $\bf{Z}$ for capturing the correlation structure of the data and uses $\ell_{21}$ norm to describe the matrix $\bf{E}$. Based on SSC and LRR, many works \cite{lu2013correlation, 7342943, jiang2016robust, li2015structured, 7460141, DBLP:conf/wacv/JiSL14} were proposed to design different regularizations for the representation matrix $\bf{Z}$ and choose a simple norm on the noise matrix $\bf{E}$.

Note that the previous works mainly focus on choosing a proper norm to regularize the representation matrix and ignore the influence of the noise term on subspace clustering. However, the real data are always contaminated by the unknown noise, and the noise usually has a complicated statistical distribution \cite{liu2013robust,li2015subspace,lee2015membership}. If we can't adopt a proper model to deal with the noise, the learned representation matrix may fail to capture the similarity between samples which can result in a unreliable subspace clustering result.
So how to handle the noise is a difficult task and has a significant influence on subspace clustering. Although the existing methods choose the different norm to handle the noise, they can only deal with the specific noise. For example, $\ell_1$ norm is suitable for entry-wise corruptions, $\ell_{21}$ norm is for sample-specific corruptions and Frobenius norm is to tackle Gaussian noise. Besides, Li et al. \cite{li2015subspace} tried to describe the noise using Mixture of Gaussian Regression (MoG Regression). Although it has shown its superiority through the comparison experiments, it is sensitive to the number of Gaussian and has high computational cost.

To alleviate the noise's effect on subspace clustering, we in this paper propose a subspace clustering method by using Cauchy loss function (CLF) to suppress the noise term. Compared with the conventional $\ell_1$ or $\ell_2$ loss, the influence function of CLF has a upper bound. So it can alleviate the influence of a single sample, especially the sample with a large noise, on estimating the residuals. Therefore, CLF has less dependence on the distribution of the noise and is more robust to the noise. Because our work mainly focuses on the noise term, we simply use the Frobenius norm to regularize the representation matrix. Furthermore, we prove the grouping effect of our method, which means that highly correlated data can be grouped together. Experimental results on the real datasets show the effectivness of our proposed method.

\subsection{Paper Contributions and Organization}

Our work has the following three main contributions.

\begin{enumerate}
\item We propose a robust subspace clustering method based on Cauchy loss function (CLF). Specifically, CLF is able to penalize the point with large noise rather than giving a specific assumption on the distribution of the noise. So our method is more robust to different kinds of the noise in the real data.
\item The grouping effect of our method is theoretically proved, which can preserve the local structure in the raw data. Therefore, highly correlated point can be grouped together in the low-dimensional subspace.
\item We verify our method on different real applications, including motion segmentation and image clustering. The experimental results show that our method achieves better performance than several representative methods.
\end{enumerate}

The rest of this paper is arranged as below: The related work are introduced in Section \ref{section 2}. Section \ref{section 3} gives the problem formulation and the whole framework of our subspace clustering algorithm. In Section \ref{section 4}, we prove the grouping effect of our method which is a very useful property for subspace clustering, and then analyze the convergence of our optimization algorithm. The experimental results on real databases are presented in Section \ref{section 5}. Finally, the paper is briefly concluded in Section \ref{section 6}.

\section{Related Work \label{section 2}}
Considering that our proposed method is a kind of spectral clustering based method, we mainly review the most recent and related works. Throughout the paper, we use the non-bold letters, bold lower case letters and bold upper case letters to represent scalars, vectors and matrices respectively.

Sparse Subspace Clustering (SSC) \cite{elhamifar2009sparse}, as a first proposed spectral clustering based method, aims to find the sparsest representation for each point with all other points in a union of subspaces by solving the following problem:
\begin{equation}\label{2}
\begin{array}{l}
\mathop {\min }\limits_{\bf{Z},\bf{E}} \left\| \bf{E} \right\|_F^2 + \lambda {\left\| \bf{Z} \right\|_0}\\
s.t.{\kern 1pt} {\kern 1pt} {\kern 1pt} {\kern 1pt} {\bf{X} = \bf{XZ} + \bf{E}},{\kern 1pt} {\kern 1pt} {\kern 1pt} diag({\bf{Z}}) = \bf{0},
\end{array}
\end{equation}
where $\lambda>0$ is a weighting factor to balance two terms. $diag(\bf{Z}) = 0$ is used to avoid the solution $\bf{Z}$ being an identity matrix, which means that one point can not be reconstructed using itself.  As we all known, solving such sparse representation is a NP hard problem. So SSC uses $\ell_1$ norm to approximate the $\ell_0$ norm. The final objective function is given below:
\begin{equation}\label{3}
\begin{array}{l}
\mathop {\min }\limits_{\bf{Z},\bf{E}} \left\| \bf{E} \right\|_F^2 + \lambda {\left\| \bf{Z} \right\|_1}\\
s.t.{\kern 1pt} {\kern 1pt} {\kern 1pt} {\kern 1pt} {\bf{X = XZ + E}},{\kern 1pt} {\kern 1pt} {\kern 1pt} diag(\bf{Z}) = 0.
\end{array}
\end{equation}
SSC assumes that one point can be reconstructed only using few points in the same subspace. When the data are drawn from independent subspaces, SSC can divide the points into their subspaces. But for the real data, the representation matrix of SSC may be too sparse to capture the relationship between points in the same subspace. Based on SSC, Wang and Xu \cite{DBLP:journals/jmlr/WangX16} proposed a modified version, named Noisy Sparse Subspace Clustering (NSSC), to deal with noisy data.

Low-Rank Representation (LRR) \cite{liu2013robust} was proposed to capture the correlation structure of the data by finding a low-rank representation of the samples instead of a sparse one. The original problem of LRR is formulated as
\begin{equation}\label{4}
\begin{array}{l}
\mathop {\min }\limits_{\bf{Z}} {\kern 1pt} {\kern 1pt} {\kern 1pt} {\kern 1pt} rank(\bf{Z})\\
s.t.{\kern 1pt} {\kern 1pt} {\kern 1pt} {\kern 1pt} \bf{X = XZ}.
\end{array}
\end{equation}
The above optimization problem is hard to be solved due to the discrete nature of the rank function. So LRR adopts the nuclear norm as a surrogate of the rank function. Furthermore, LRR uses $\ell_{21}$ norm to deal with the noise term for improving its robustness to the noise and outliers. The subspace clustering problem becomes
\begin{equation}\label{5}
\begin{array}{l}
\mathop {\min }\limits_{\bf{Z}} {\kern 1pt} {\kern 1pt} {\kern 1pt} {\kern 1pt} {\left\| \bf{E} \right\|_{21}} + \lambda {\left\| \bf{Z} \right\|_*}\\
s.t.{\kern 1pt} {\kern 1pt} {\kern 1pt} {\kern 1pt} \bf{X = XZ + E}.
\end{array}
\end{equation}
However, there is no theoretical analysis about the importance of low rank property of the representation matrix $\bf{Z}$ for subspace clustering. Besides, the solution $\bf{Z}^*$ may be very dense and far from block-diagonal.

Least Squares Regression (LSR) \cite{lu2012robust} employs the Frobenius norm to handle the representation matrix and the noise matrix simultaneously. The corresponding optimization problem is defined as
\begin{equation}\label{6}
\begin{array}{l}
\mathop {\min }\limits_{\bf{Z}} {\kern 1pt} {\kern 1pt} {\kern 1pt} {\kern 1pt} {\left\| \bf{E} \right\|_F^2} + \lambda {\left\| \bf{Z} \right\|_F^2}\\
s.t.{\kern 1pt} {\kern 1pt} {\kern 1pt} {\kern 1pt} \bf{X = XZ + E}.
\end{array}
\end{equation}
Note that the above problem can be efficiently solved. The main contribution of LSR is that it encourages grouping effect which can group highly correlated data together.

In order to balance the sparsity and low rank property of the representation matrix, Correlation Adaptive Subspace Segmentation (CASS) \cite{lu2013correlation} was proposed to optimize the problem
\begin{equation}\label{7}
\begin{array}{*{20}{l}}
{\mathop {\min }\limits_{{\bf{Z}},{\bf{E}}} \left\| {\bf{E}} \right\|_F^2 + \lambda \sum\limits_{i = 1}^n {{{\left\| {{\bf{X}}diag({{\bf{z}}_i})} \right\|}_*}} }\\
{s.t.{\kern 1pt} {\kern 1pt} {\kern 1pt} {\kern 1pt} {\bf{X}} = {\bf{XZ}} + {\bf{E}},}
\end{array}
\end{equation}
where ${{{\left\| {{\bf{X}}diag({{\bf{z}}_i})} \right\|}_*}}$ is trace lasso and its definition can be found in \cite{lu2013correlation}. Due to taking the data correlation into account, it can adaptively interpolate SSC and LSR.
\begin{figure*}[t]
\begin{center}
\includegraphics[width=14cm]{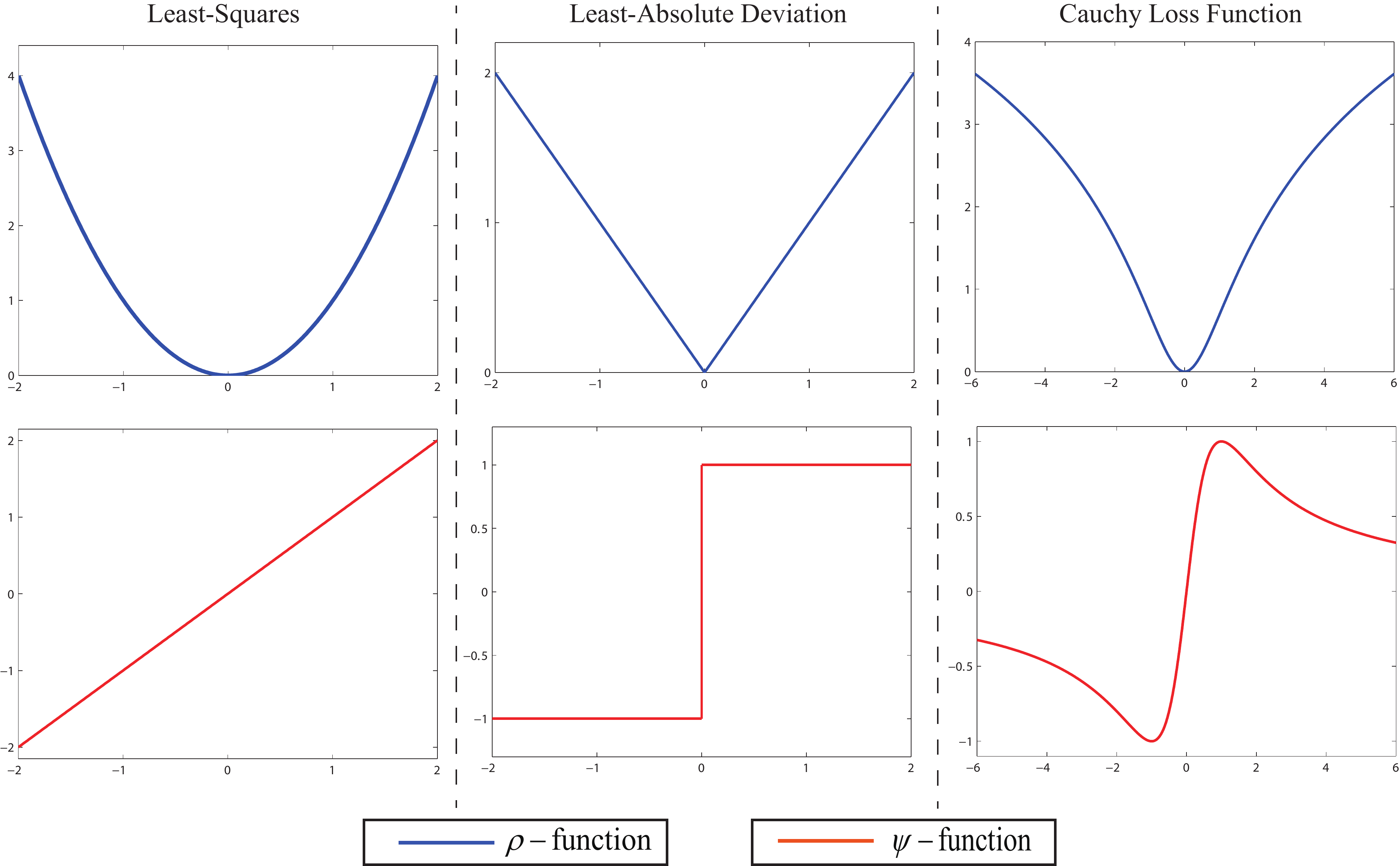}\\
\end{center}
\caption{An illustration of different estimators. The right is least-squares, the middle represents least-absolute deviation and the left is Cauchy loss function.}\label{Cauchy}
\end{figure*}

Mixture of Gaussian Regression (MoG Regression) \cite{li2015subspace}, as a most related method to our work, uses the mixture of Gaussian model to describe the noise term and tries to solve the following problem
\begin{equation}\label{8}
\begin{split}
&\mathop {\min }\limits_{{\bf{Z,E,\pi ,\Sigma }}}  - \sum\limits_{i = 1}^n {\ln \left( {\sum\limits_{k = 1}^K {{\pi _k}N({{\bf{e}}_i}|0,{{\bf{\Sigma }}_k})} } \right) + \lambda \left\| {\bf{Z}} \right\|_F^2}  \\
&s.t.{\kern 1pt} {\kern 1pt} {\kern 1pt} {\bf{X = XZ + E}},diag(\bf{Z}) = 0,\\
&{\kern 1pt} {\kern 1pt} {\kern 1pt} {\kern 1pt} {\kern 1pt} {\kern 1pt} {\kern 1pt} {\kern 1pt} {\kern 1pt} {\kern 1pt} {\kern 1pt} {\kern 1pt} {\kern 1pt} {\kern 1pt} {\kern 1pt} {\pi _k} \ge 0,{{\bf{\Sigma}} _k} \in {S^ + },\sum\limits_{k = 1}^K {{\pi _k} = 1} ,
\end{split}
\end{equation}
where ${\kern 1pt} {\pi _k}$ is the mixing weight, ${{{\bf{e}}_n}}$ is mean vector, ${{{\bf{\Sigma}} _k}}$ is the covariance matrix and $K$ denotes the number of Gaussian. Although MoG Regression has better performance than the single Gaussian model, it is only a extended version of single Gaussian and is sensitive to the number of Gaussian. Additionally, solving the above problem needs high computation cost.

\section{Subspace Clustering by CLF \label{section 3}}
In this paper, we propose a new spectral clustering based method to alleviate the influence of the noise on subspace clustering. Particularly, we employ Cauchy loss function (CLF) to suppress the noise. Next we give the details of our optimization objection function and the framework of our subspace clustering method.
\subsection{Problem Formulation}
In statistics, M-estimator is a broad class of estimators, which is used to represent the minima of sum of functions. Let $r_i$ denotes the residual of the $i$-th data with its estimated value and $\rho(r_i)$ be a symmetric and positive-define function which has a unique minimum at zero. M-estimator aims to optimize the following problem:
\begin{equation}\label{9}
\min \sum\limits_i {\rho ({r_i}).}
\end{equation}
The influence function of $\rho$-function is defined as:
\begin{equation}\label{10}
\psi (x) = \frac{{\partial \rho (x)}}{{\partial x}},
\end{equation}
which is used to measure the effect of changing a point of the sample on the value of the parameter estimation.

We demonstrate different estimators and their influence functions in Fig. \ref{Cauchy}. For the $l_2$ estimator (least-squares) with $\rho(x)=x^2$, its influence function is $\psi(x)=x$. From Fig. \ref{Cauchy}, we can see that the influence of a sample on the estimate grows linearly as the error increases. This means the $l_2$ estimator is not robust to the noise. Although the $l_1$ estimator (least-absolute deviation) with $\rho(x)=\left| x \right|$ can alleviate the effect of the large error, its influence function has no cut-off \cite{xu2015multi,he2000breakdown}. For a robust estimator, its influence function should not be sensitive to the increase of the error. CLF gives good characteristic on this aspect, and its definition is shown below
\begin{equation}\label{11}
\rho (x) = \log (1 + {(x/c)^2})
\end{equation}
with influence function
\begin{equation}\label{12}
\psi (x) = \frac{{2x}}{{{x^2} + {c^2}}},
\end{equation}
where $c$ is a constant. Note that CLF's influence function has the upper bound and its value tends to zero with the increase of the error.

Considering CLF is robust to the noise, we use CLF to penalize the noise term which is defined as
\begin{equation}\label{13}
\sum\limits_{i = 1}^n {\log (1 + \frac{{\left\| {{{\bf{x}}_i} - {\bf{X}}{{\bf{z}}_i}} \right\|_2^2}}{{{c^2}}}),}
\end{equation}
where $\bf{X}$ is the data matrix, and ${\bf{z}}_i$ denotes the representation vector of the $i$-th data ${\bf{x}}_i$. As stated before, we simply use the Frobenius norm to regularize the representation matrix for verifying the influence of the noise model on subspace clustering and facilitating the problem solving. The corresponding model can be formulated as
\begin{equation}\label{14}
\mathop {\min }\limits_{\bf{Z}} \sum\limits_{i = 1}^n {\log (1 + \frac{{\left\| {{{\bf{x}}_i} - {\bf{X}}{{\bf{z}}_i}} \right\|_2^2}}{{{c^2}}}) + \lambda \left\| {\bf{Z}} \right\|_F^2,}
\end{equation}
where $\lambda$ is a weight factor to balance the effect of two terms. For the formula (\ref{14}), an iterative algorithm can be employed to find the solution for each data point, but it is not a high-efficiency way to obtain the representation matrix. In order to reduce the time complexity and keep the valuable property, we revise the formula (\ref{14}) and give the final objective function
\begin{equation}\label{15}
\mathop {\min }\limits_{\bf{Z}} \log (1 + \frac{{\left\| {{\bf{X}} - {\bf{XZ}}} \right\|_F^2}}{{{c^2}}}) + \lambda \left\| {\bf{Z}} \right\|_F^2.
\end{equation}
Note that it takes the representation matrix $\bf{Z}$ as an integrate to learn. Therefore we can directly to optimize the representation matrix by using an iteration process.
\subsection{Optimization}
For the problem (\ref{15}), we adopt Iteratively Re-weighted Residuals (IRR) method to find the solution. Given the data matrix $X$, the formula (\ref{15}) can be rewritten as
\begin{equation}\label{16}
\mathop {\min }\limits_{\bf{Z}} {\cal J} = \log (1 + \frac{{\left\| {{\bf{X - XZ}}} \right\|_F^2}}{{{c^2}}}) + \lambda \left\| {\bf{Z}} \right\|_F^2.
\end{equation}
Setting the derivative of $\mathcal{J}$ with respect to $\bf{Z}$ to zero, we have
\begin{equation}\label{17}
\frac{{ - 2{{\bf{X}}^T}({\bf{X}} - {\bf{XZ}})}}{{{c^2} + \left\| {{\bf{X}} - {\bf{XZ}}} \right\|_F^2}} + 2\lambda {\bf{Z}} = 0,
\end{equation}
which is equivalent to
\begin{equation}\label{18}
\left( {\frac{{{{\bf{X}}^T}{\bf{X}}}}{{{c^2} + \left\| {{\bf{X}} - {\bf{XZ}}} \right\|_F^2}} + \lambda {\bf{I}}} \right){\bf{Z}} = \frac{{{{\bf{X}}^T}{\bf{X}}}}{{{c^2} + \left\| {{\bf{X}} - {\bf{XZ}}} \right\|_F^2}}.
\end{equation}
Then we can obtain the solution
\begin{algorithm}[t]
  \caption{Iteratively Re-weighted Residuals}
  \label{alg1}
  \vspace{0.1cm}
  \textbf{Input:}\ data matrix $\bf{X}$, parameters $\lambda$ and $c$, initial representation matrix ${\bf{Z}}^0$, $t=0$.

  \textbf{Output:}\ ${\bf{Z}}^*$.

  \hspace{0.2cm}\textbf{while} not converge \textbf{do}
  \begin{enumerate}
    \item ${{\bf{R}}^{t + 1}} \leftarrow {\bf{X}} - {\bf{X}}{{\bf{Z}}^t}$
    \item ${{Q}^{t + 1}} \leftarrow 1/({c^2} + \left\| {{{\bf{R}}^{t + 1}}} \right\|_F^2)$
    \item ${{\bf{Z}}^{t + 1}} \leftarrow {Q}^{t + 1}{\left( {{Q}^{t + 1}}{{{\bf{X}}^T}{\bf{X}} + 2\lambda {\bf{I}}} \right)^{ - 1}}{{\bf{X}}^T}{\bf{X}}$
  \end{enumerate}
  \hspace{0.2cm}\textbf{end while}
\end{algorithm}
\begin{equation}\label{19}
\left\{ {\begin{split}
&{\bf{Z}} = Q{\left( Q{{{\bf{X}}^T}{\bf{X}} + \lambda {\bf{I}}} \right)^{ - 1}}{{\bf{X}}^T}{\bf{X}}\\
& Q = \frac{1}{{{c^2} + \left\| {\bf{R}} \right\|_F^2}}\\
&{\bf{R}}{\rm{  =  }}{\bf{X}}{\rm{ - }}{\bf{XZ}}&
\end{split}} \right.,
\end{equation}
where $\bf{R}$ is the residual of the data matrix with the corrected matrix, and $Q$ is the weight function which is used to reduce the effect of the noise.
Note that $Q$ should be calculated using the representation matrix $\bf{Z}$. Then an iterative way is adopted to update $\bf{Z}$ until convergence. The whole procedure for solving problem (\ref{15}) is described in Algorithm 1.
\subsection{Subspace Clustering Algorithm via CLF}

In this section, we give the framework of our proposed subspace clustering algorithm which is outlined in Algorithm \ref{alg2}. Note that we first use Algorithm \ref{alg1} to find the representation matrix ${\bf{Z}}^*$. Then the similarity matrix is defined as ${\bf{W}} = (\left| {{{\bf{Z}}^*}} \right| + | {{{({{\bf{Z}}^*})}^T}} |)/2$, where ${({{\bf{Z}}^*})^T}$ is the transposition of ${\bf{Z}}^*$. Finally, Normalized Cuts \cite{shi2000normalized}, a kind of spectral clustering algorithm \cite{Luxburg2007A}, is employed to group the data points into $k$ clusters based on the similarity matrix.

\begin{figure*}[!t]
\centering
\subfigure{\includegraphics[width=17.45cm]{{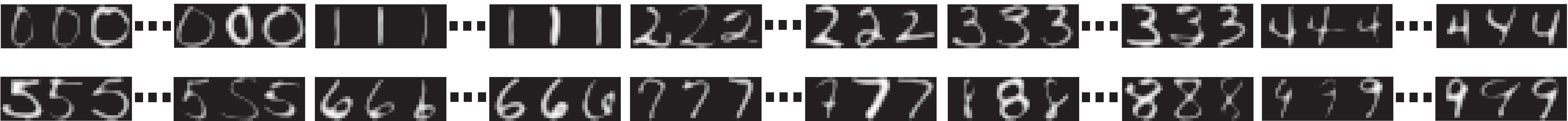}}%
\label{fig_first_case}}
\hfil
\subfigure[SSC]{\includegraphics[width=4cm]{{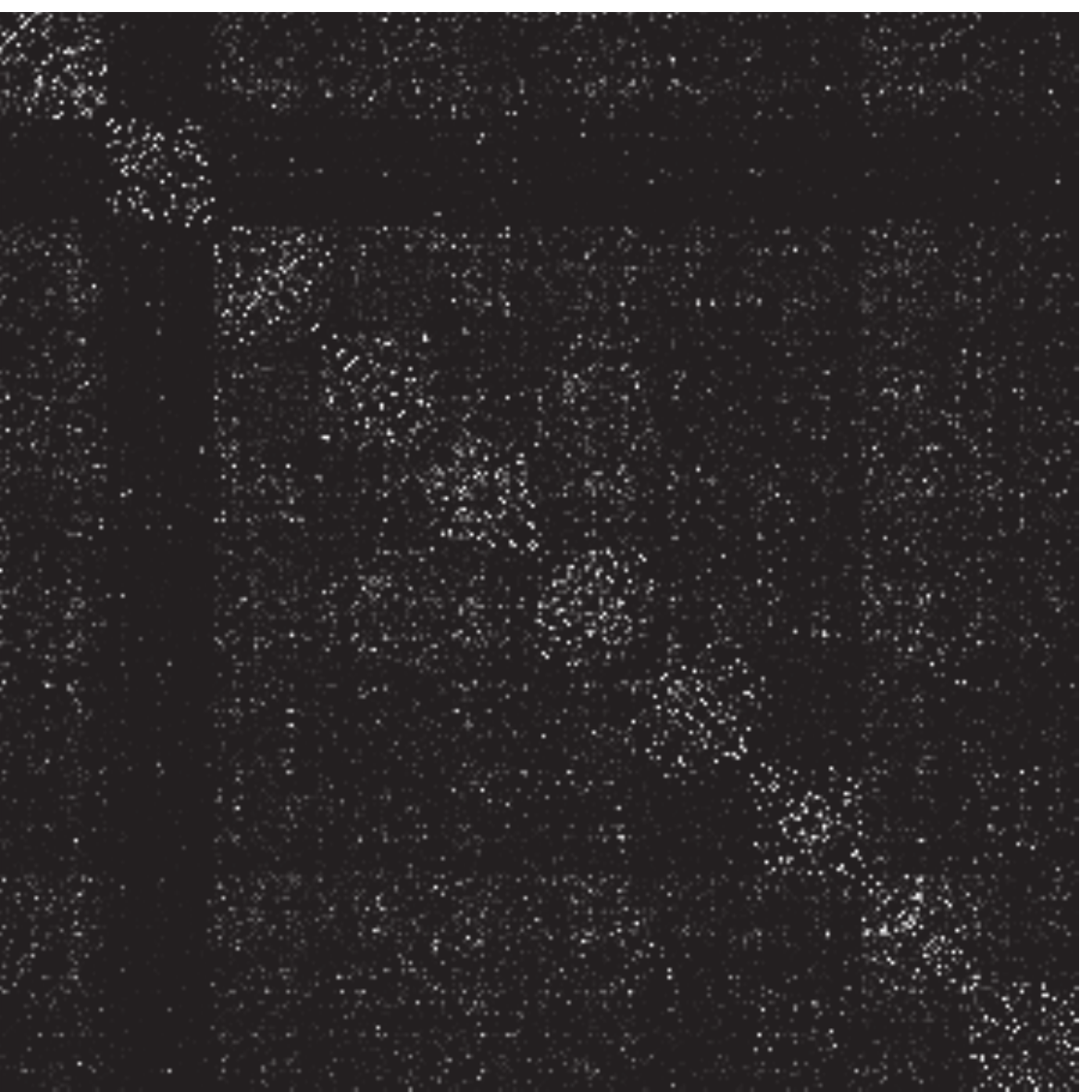}}%
\label{fig_second_case}}
\hfil
\subfigure[LRR]{\includegraphics[width=4cm]{{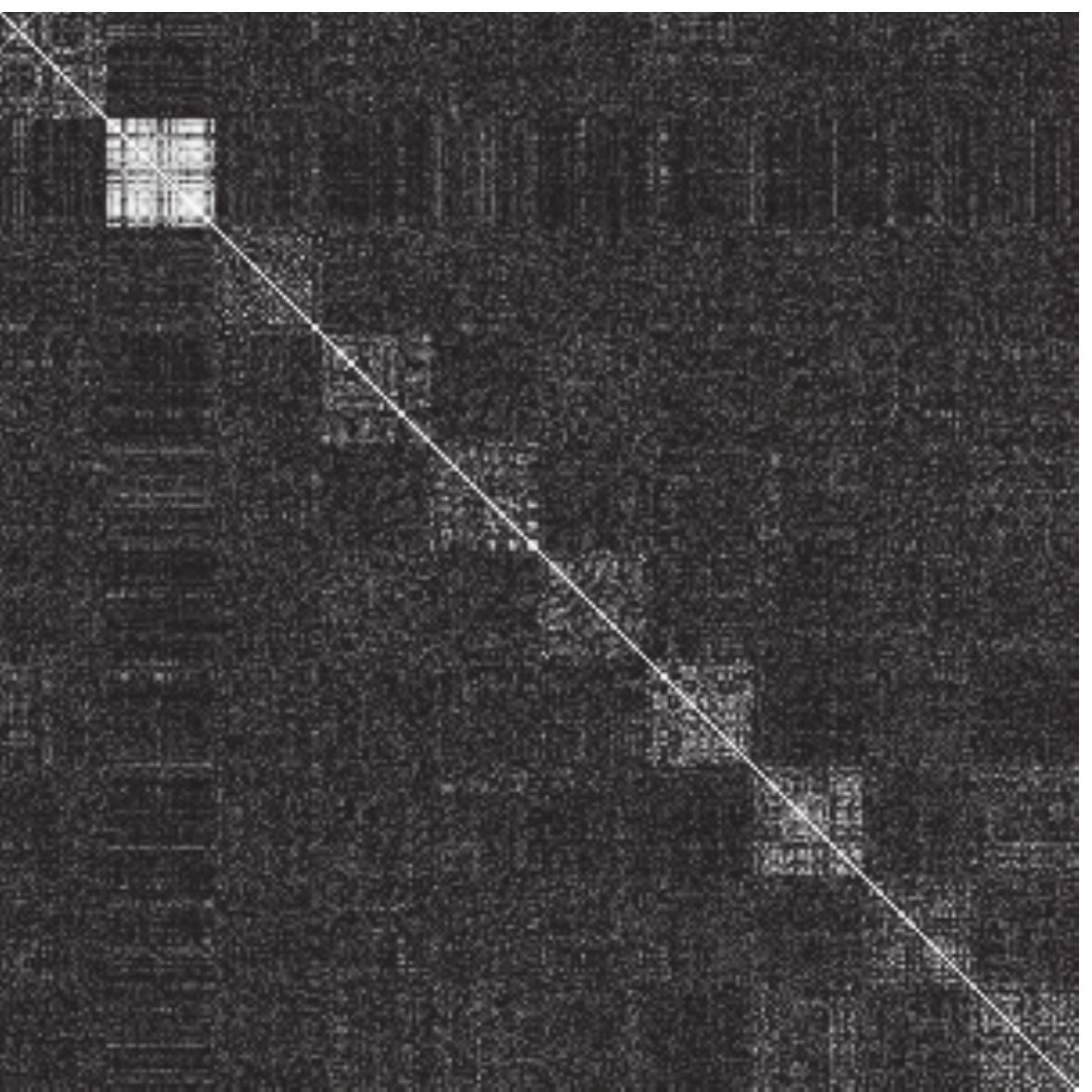}}%
\label{fig_second_case}}
\hfil
\subfigure[LSR]{\includegraphics[width=4cm]{{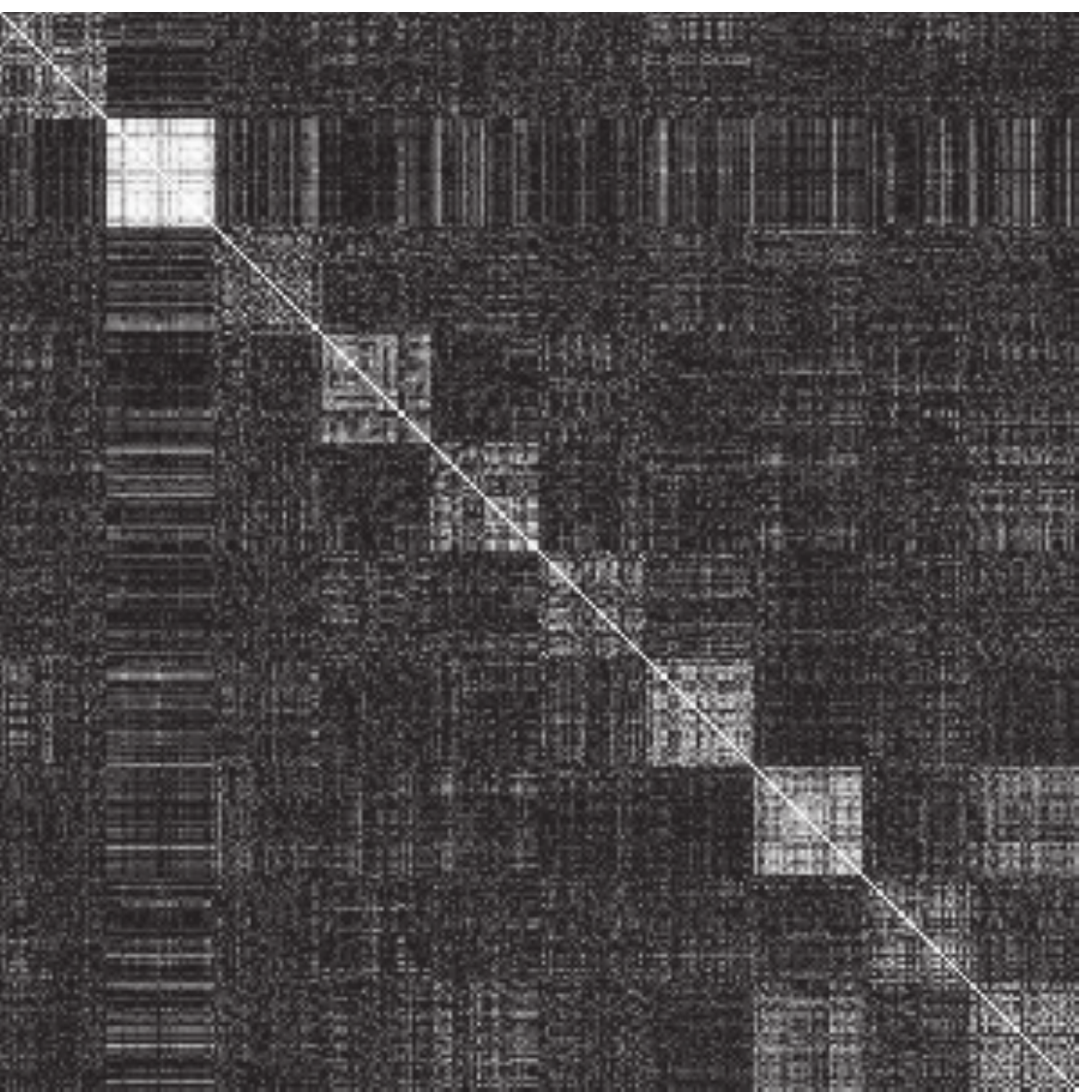}}%
\label{fig_second_case}}
\hfil
\subfigure[CASS]{\includegraphics[width=4cm]{{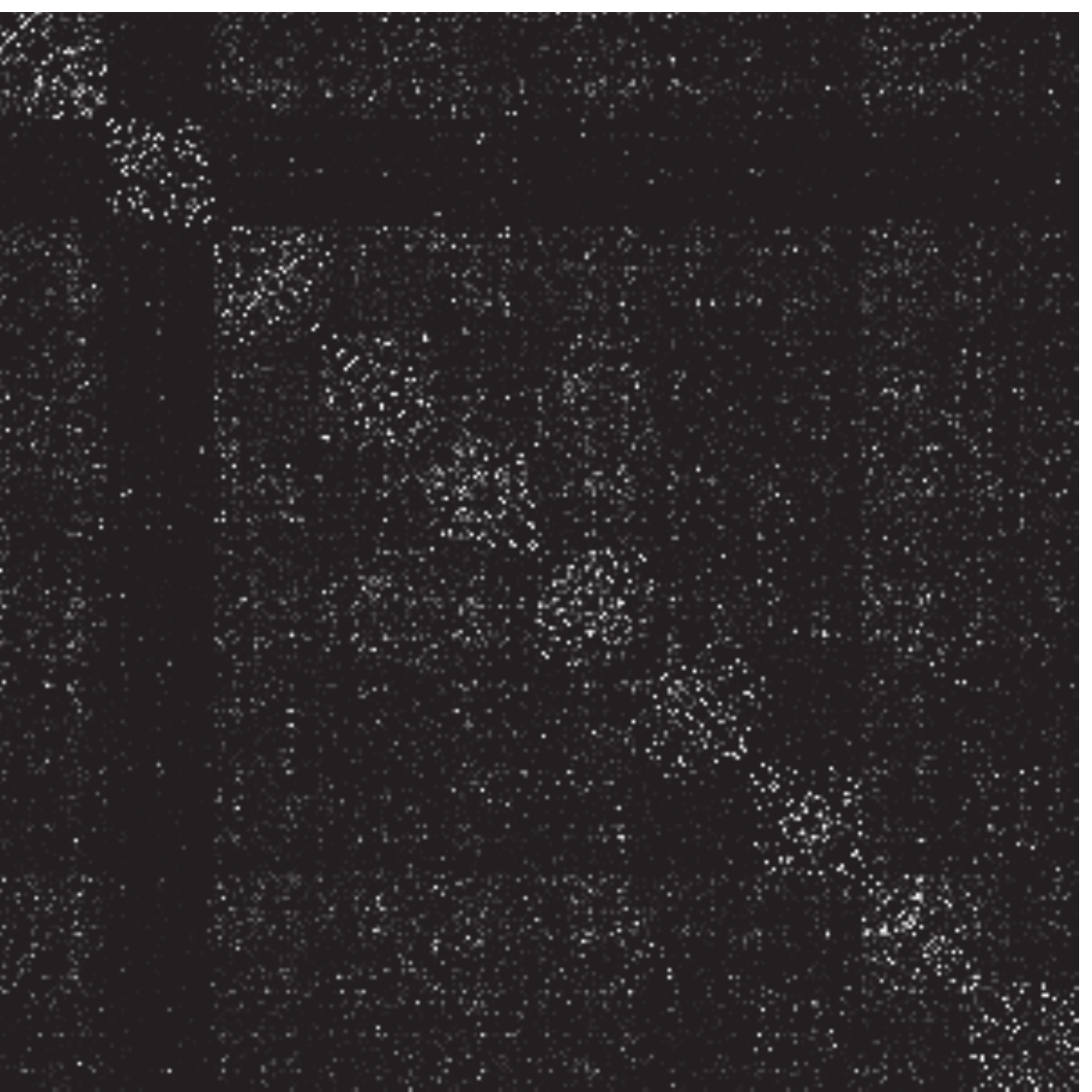}}%
\label{fig_second_case}}
\hfil
\subfigure[MoG]{\includegraphics[width=4cm]{{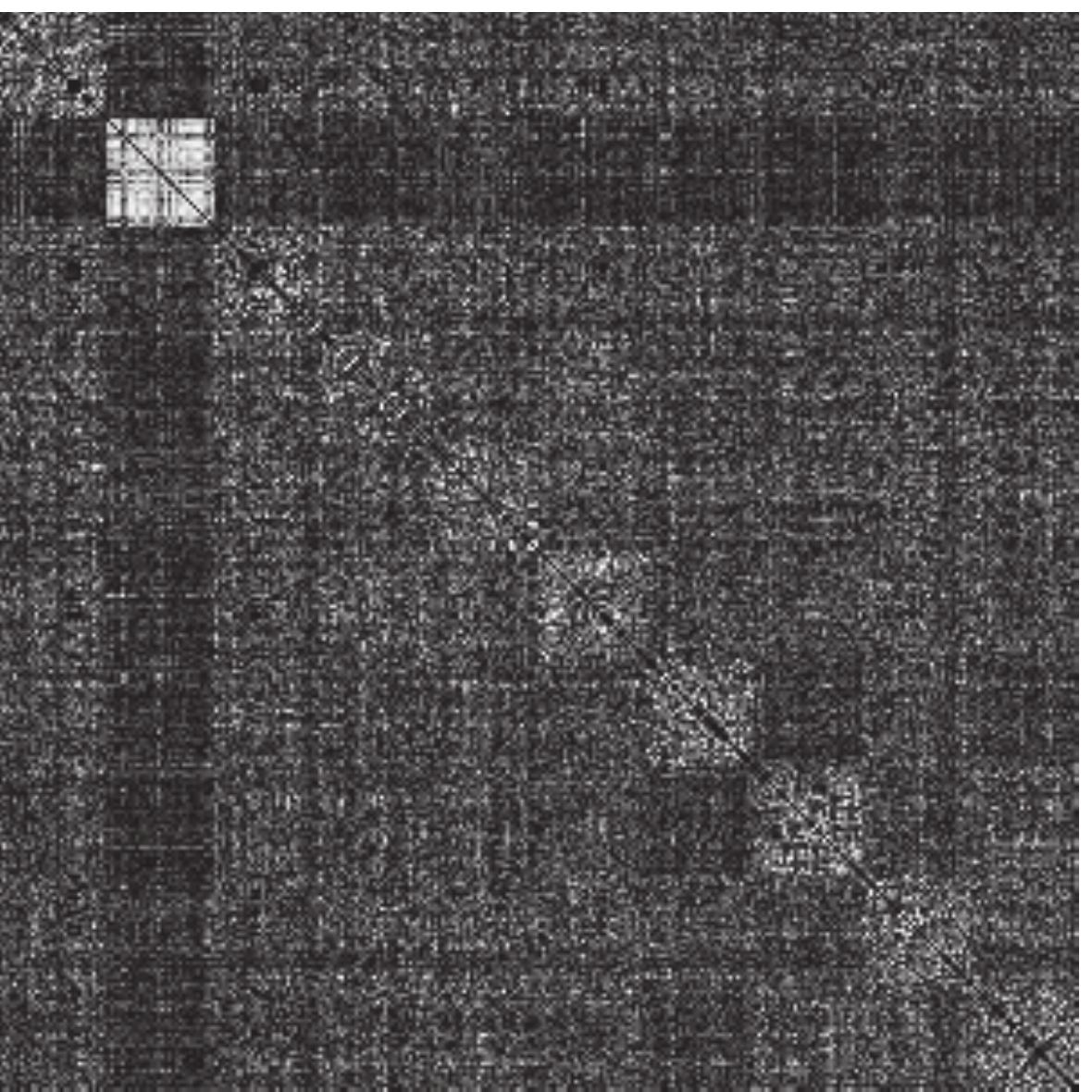}}%
\label{fig_second_case}}
\hfil
\subfigure[NSSC]{\includegraphics[width=4cm]{{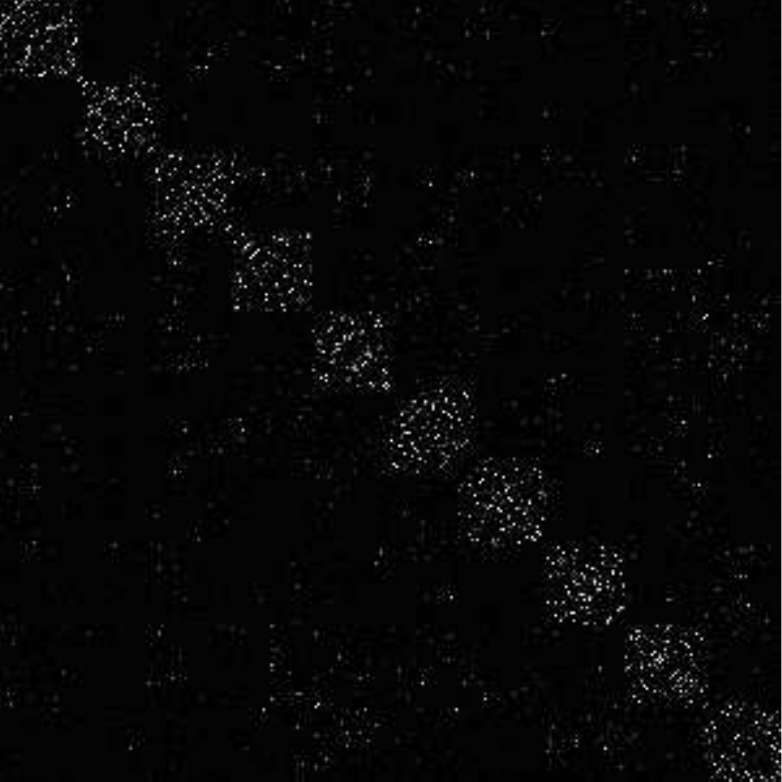}}%
\label{fig_second_case}}
\hfil
\subfigure[Ours]{\includegraphics[width=4cm]{{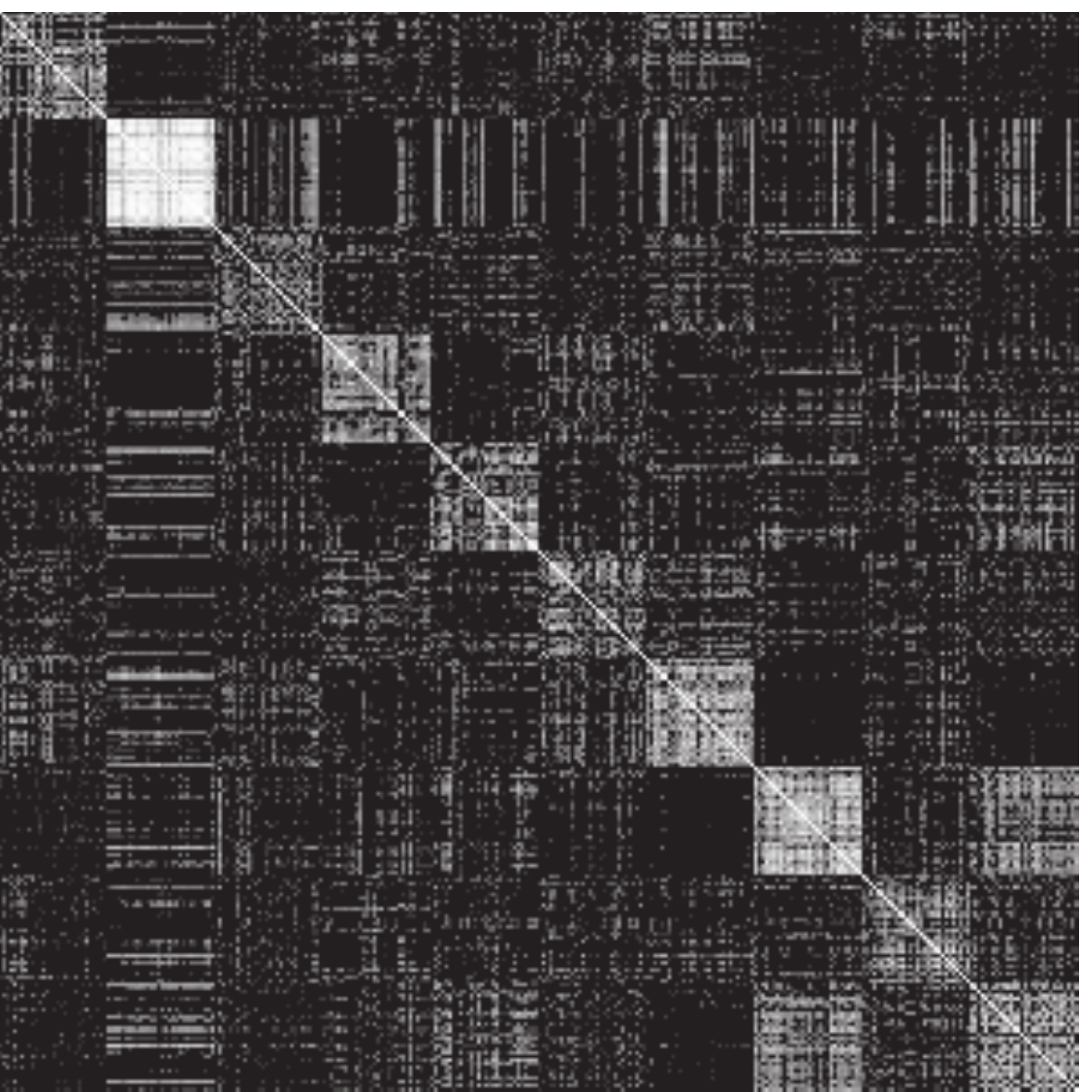}}%
\label{fig_second_case}}
\caption{An example of the similarity matrix $W$ of 10 classes derived by different methods on the USPS database.}
\label{AffinityMatrix}
\end{figure*}
In order to demonstrate the structure of the learned similarity matrix, we show the similarity matrices of 10 subjects derived by SSC, LRR, LSR, CASS, MoG Regression, NSSC and our proposed method on the USPS dataset in Fig. \ref{AffinityMatrix}. For simplicity, we use MoG to denote MoG Regression. USPS is a popular handwritten digit database for clustering analysis. From Fig. \ref{AffinityMatrix}, we can see that all the methods can give a approximate block-diagonal matrix. The similarity matrices obtained by SSC and CASS are sparse and similar which means that CASS gives a large weight for the sparsity of the representation matrix. Besides, NSSC also gives a very sparse similarity matrix. However, the points in the same cluster have no high correlation which can degenerate the performance of subspace clustering. In contrast, the similarity matrices learned by LRR, LSR, MoG Regression and our method are very dense which give high similarity for the samples within the same cluster. Furthermore, we define a Contrast Index (CI) to quantitatively measure the difference between diagonal blocks and non-diagonal blocks of the similarity matrix. The corresponding formulation is
\begin{equation}\label{20}
CI = \frac{{{S_D}}}{{{S_D} + {S_{ND}}}} = \frac{{{S_D}}}{{{{\left\| {W} \right\|}_1}}},
\end{equation}
where $S_D$ and $S_{ND}$ denote the sum of elements in diagonal and non-diagonal blocks, respectively. Table \ref{Contrast} gives the CI of the similarity matrices obtained by different methods. Note that MoG gives a lowest CI which can be seen from Fig. \ref{AffinityMatrix}. Obviously, our method gives a higher CI than other methods which means that our proposed model has greater ability to group correlated data together.
\begin{algorithm}[t]
  \caption{Subspace Clustering Algorithm via CLF}
  \label{alg2}
  \begin{algorithmic}[1]
  \REQUIRE data matrix $\bf{X}$, number of subspaces $k$
  \begin{enumerate}[leftmargin=1pt]
    \item Solve the problem (\ref{15}) and obtain the final representation matrix ${\bf{Z}}^*$.
    \item Construct similarity matrix W using $(\left| {\bf{Z}}^* \right| + \left| {{({\bf{Z}}^*)^T}} \right|)/2$.
    \item Group the data points into $k$ clusters by Normalized Cuts.
  \end{enumerate}
  \end{algorithmic}
\end{algorithm}

\begin{table}[t]
\renewcommand{\arraystretch}{1.3}
\caption{The Contrast Index (CI) (\%) of the similarity matrices obtained by different methods.}\label{Contrast}
\begin{center}
\begin{tabular}{|c|c|c|c|c|c|c|c|}
\hline
Method & SSC & LRR & LSR & CASS & MoG & NSSC & Ours\\
\hline
CI & 30.39 & 23.89 & 25.46 & 30.30 & 16.80 & 32.39 &\textbf{ 38.12}\\
\hline
\end{tabular}
\end{center}
\end{table}
\section{Theoretical Analysis \label{section 4}}
In this section, we prove that our proposed method has the grouping effect which can group highly correlated data together, and then analyze the convergence of our optimization algorithm.
\subsection{The Grouping Effect}
\begin{myTheo}
\label{Theo1}
Given a data point ${\bf{x}}\in\mathbb{R}^d$, the normalized data matrix $\bf{X}$ and a parameter $\lambda$. Let ${\bf{\hat z}}$ be the optimal solution to the following problem (in vector form):
\begin{equation}\label{21}
\mathop {\min }\limits_{\bf{z}} \log (1 + \frac{{\left\| {{\bf{x}} - {\bf{Xz}}} \right\|_2^2}}{{{c^2}}}) + \lambda \left\| {\bf{z}} \right\|_2^2.
\end{equation}
Then we have
\begin{equation}\label{22}
\frac{{\left| {{{{\hat z}}^i} - {{{\hat z}}^j}} \right|}}{{{{\left\| {\bf{x}} \right\|}_2}}} \le \frac{1}{{\lambda {c^2}}}\sqrt {2(1 - r)} ,
\end{equation}
where $r = {\bf{x}}_i^T{{\bf{x}}_j}$ is the sample correlation. ${{{\hat z}}^i}$ and ${{{\hat z}}^j}$ are the $i$-th and $j$-th entries of vector ${\bf{\hat z}}$. ${\bf{x}}_i$ and ${\bf{x}}_j$ are the $i$-th and $j$-th columns of $\bf{X}$.
\begin{proof}
Let
\begin{equation}\label{23}
L({\bf{z}}) = \log (1 + \frac{{\left\| {{\bf{x}} - {\bf{Xz}}} \right\|_2^2}}{{{c^2}}}) + \lambda \left\| {\bf{z}} \right\|_2^2.
\end{equation}
Since ${\bf{\hat z}} = \mathop {\arg \min }\limits_{\bf{z}} L({\bf{z}})$, we have
\begin{equation}\label{24}
{\left. {\frac{{\partial L({\bf{z}})}}{{\partial {\bf{z}}}}} \right|_{{\bf{z}} = {\bf{\hat z}}}} = 0.
\end{equation}
This gives
\begin{equation}\label{25}
\frac{{ - 2{{\bf{x}}_i}^T({\bf{x}} - {\bf{X\hat z}})}}{{{c^2} + \left\| {{\bf{x}} - {\bf{X\hat z}}} \right\|_2^2}} + 2\lambda {{\hat z}^i} = 0,
\end{equation}
\begin{equation}\label{26}
\frac{{ - 2{{\bf{x}}_j}^T({\bf{x}} - {\bf{X\hat z}})}}{{{c^2} + \left\| {{\bf{x}} - {\bf{X\hat z}}} \right\|_2^2}} + 2\lambda {\hat z^j} = 0,
\end{equation}
Equations (\ref{25}) and (\ref{26}) give
\begin{equation}\label{27}
{\hat z^i} - {\hat z^j} = \frac{{({{\bf{x}}_i}^T - {{\bf{x}}_j}^T)({\bf{x}} - {\bf{X\hat z}})}}{{\lambda ({c^2} + \left\| {{\bf{x}} - {\bf{X\hat z}}} \right\|_2^2)}} \le \frac{{({{\bf{x}}_i}^T - {{\bf{x}}_j}^T)({\bf{x}} - {\bf{X\hat z}})}}{{\lambda {c^2}}}.
\end{equation}
Since each column of $\bf{X}$ is normalized, ${\left\| {{{\bf{x}}_i} - {{\bf{x}}_j}} \right\|_2} = \sqrt {2(1 - r)}$, where $r = {{\bf{x}}_i}^T{{\bf{x}}_j}$. Note that $\bf{\hat z}$ is the optimal to the problem (\ref{21}), and we deduce
\begin{equation}\label{28}
\begin{split}
\log (1 + \frac{{\left\| {{\bf{x}} - {\bf{X\hat z}}} \right\|_2^2}}{{{c^2}}}) \le \log (1 + \frac{{\left\| {{\bf{x}} - {\bf{X\hat z}}} \right\|_2^2}}{{{c^2}}}) + \lambda \left\| {{\bf{\hat z}}} \right\|_2^2\\
 = L({\bf{\hat z}}) \le L(0) = \log (1 + \frac{{\left\| {\bf{x}} \right\|_2^2}}{{{c^2}}}).
\end{split}
\end{equation}
Thus ${\left\| {{\bf{x}} - {\bf{X\hat z}}} \right\|_2} \le {\left\| {\bf{x}} \right\|_2}$. Finally, we obtain
\begin{equation}\label{29}
\frac{{\left| {{{\hat z}^i} - {{\hat z}^j}} \right|}}{{{{\left\| {\bf{x}} \right\|}_2}}} \le \frac{1}{{\lambda {c^2}}}\sqrt {2(1 - r)} .
\end{equation}
\end{proof}
\end{myTheo}

As stated in Theorem \ref{Theo1}, if ${\bf{x}}_i$ and ${\bf{x}}_j$ are highly correlated, the value of $r$ is close to 1, which means that the difference between ${\hat z}^i$ and ${\hat z}^j$ is almost 0. Then ${\bf{x}}_i$ and ${\bf{x}}_j$ can be grouped into the same cluster. Note that Theorem \ref{Theo1} gives the grouping effect for one point (vector form). For the matrix form, the corresponding grouping effect can still be proved using the similar proof procedure of Theorem \ref{Theo1}.
\subsection{Convergence Analysis}
We employ the Weiszfeld's method \cite{voss1980linear} to analyze the convergence of Algorithm \ref{alg1}. The formula (\ref{16}) is equivalent to
\begin{equation}\label{30}
\mathop {\min }\limits_{{{\bf{z}}_1},...,{{\bf{z}}_n}} {\cal J}({\bf{Z}}) = \log \left( {1 + \frac{{\sum\limits_{i = 1}^n {\left\| {{{\bf{x}}_i} - {\bf{X}}{{\bf{z}}_i}} \right\|_2^2} }}{{{c^2}}}} \right) + \lambda \sum\limits_{i = 1}^n {\left\| {{{\bf{z}}_i}} \right\|_2^2} ,
\end{equation}
where ${\bf{z}}_i$ is the representation vector of ${\bf{x}}_i$. The solution $\bf{Z}$ in (\ref{19}) can be rewritten as
\begin{equation}\label{31}
{{\bf{z}}_i} = Q{\left( Q{{{\bf{X}}^T}{\bf{X}} + 2\lambda {\bf{I}}} \right)^{ - 1}}{{\bf{X}}^T}{{\bf{x}}_i},i = 1,2,..,n.
\end{equation}
The main idea of the Weiszfeld¡¯s method is to globally approximate $\cal J$ using  a sequence of quadratic function \cite{singer2007duality}. After obtaining the solution ${\bf{Z}}^k$, we can define a upper bound of ${\cal J}({\bf{z}}_i)$ as $\phi({\bf{z}}_i;{\bf{z}}_i^k)$, where ${\cal J}({\bf{z}}_i)$ is obtained by fixing the other variables in ${\cal J}({\bf{Z}})$. $\phi({\bf{z}}_i;{\bf{z}}_i^k)$ should satisfy the following conditions:
\begin{equation}\label{32}
\begin{split}
&\phi({\bf{z}}_i^k;{\bf{z}}_i^k) = {\cal J}({\bf{z}}_i^k)\\
&\phi'({\bf{z}}_i^k;{\bf{z}}_i^k) = {\cal J}'({\bf{z}}_i^k)
\end{split}
\end{equation}
Then $\phi({\bf{z}}_i;{\bf{z}}_i^k)$ has the form
\begin{equation}\label{33}
\begin{split}
\phi ({{\bf{z}}_i};{\bf{z}}_i^k) = {\cal J}({\bf{z}}_i^k) + {({{\bf{z}}_i} - {\bf{z}}_i^k)^T}{{\cal J}^\prime }({\bf{z}}_i^k)\\
 + {({{\bf{z}}_i} - {\bf{z}}_i^k)^T}C({\bf{z}}_i^k)({{\bf{z}}_i} - {\bf{z}}_i^k)
\end{split}
\end{equation}
with symmetric matrix $C({\bf{z}}_i^k)$
\begin{equation}\label{34}
C({\bf{z}}_i^k) = \frac{{{{\bf{X}}^T}{\bf{X}}}}{{{c^2} + \left\| {{\bf{X}} - {\bf{X}}{{\bf{Z}}^k}} \right\|_F^2}} + \lambda \bf{I} .
\end{equation}
Then the convergence of Algorithm \ref{alg1} can be guaranteed by the following theorem.
\begin{myTheo}The IRR algorithm proposed in Algorithm \ref{alg1} guarantees that the objective function value of (\ref{16}) is monotone decreasing in iterations, i.e. $ {\cal J}({{\bf{Z}}^{k+1}}) \le {\cal J}({{\bf{Z}}^k}) $, until it converges.
\label{Theo2}
\end{myTheo}
\begin{proof}
Suppose that $\phi({\bf{z}}_i;{\bf{z}}_i^k)$ is locally convex with respect to ${\bf{z}}_i$ and has a local minimizer. Let ${\bf{z}}_i^{k+1}$ be the minimizer, we get
\begin{equation}\label{35}
{\phi}^\prime ({{\bf{z}}_i}^{k + 1};{\bf{z}}_i^k) = {{\cal J}^\prime }({\bf{z}}_i^k) + 2C({\bf{z}}_i^k)({{\bf{z}}_i}^{k + 1} - {\bf{z}}_i^k) = 0.
\end{equation}
Substituting for ${\cal J}'({\bf{z}}_i^k)$, we can obtain the update rule in formula (\ref{31}).

By appropriately choosing ${\bf{z}}_i^k$ near ${\bf{z}}_i$, we have ${\cal J}({\bf{z}}_i) \le \phi({\bf{z}}_i;{\bf{z}}_i^k)$ which implies that
\begin{equation}\label{36}
\begin{split}
{\cal J}({\bf{z}}_i^{k + 1}) & \le \phi ({{\bf{z}}_i}^{k + 1};{\bf{z}}_i^k)\\
 &= {\cal J}({\bf{z}}_i^k) + ({{\bf{z}}_i}^{k + 1} - {\bf{z}}_i^k)^T{\cal J}'({\bf{z}}_i^k)\\
 &+{({{\bf{z}}_i}^{k + 1} - {\bf{z}}_i^k)^T}C({\bf{z}}_i^k)({{\bf{z}}_i}^{k + 1} - {\bf{z}}_i^k).
\end{split}
\end{equation}
Equations (\ref{35}) and (\ref{36}) give
\begin{equation}\label{37}
\begin{split}
{\cal J}({\bf{z}}_i^{k + 1}) - {\cal J}({\bf{z}}_i^k) & \le - {({{\bf{z}}_i}^{k + 1} - {\bf{z}}_i^k)^T}C({\bf{z}}_i^k)({{\bf{z}}_i}^{k + 1} - {\bf{z}}_i^k)\\
 & \le  - \lambda {\left\| {{{\bf{z}}_i}^{k + 1} - {\bf{z}}_i^k} \right\|^2 \le 0}.
\end{split}
\end{equation}
So we have ${\cal J}({\bf{z}}_i^{k + 1}) \le {\cal J}({\bf{z}}_i^k)$. Based on (\ref{30}), we can easily deduce
\begin{equation}\label{38}
{\cal J}({\bf{Z}}^{k+1}) \le {\cal J}({\bf{Z}}^k).
\end{equation}
\end{proof}
\section{Experimental Verification and Analysis \label{section 5}}
In this section, we verify the effectiveness of our proposed method on five real databases: Hopkins 155 motion segmentation database \cite{tron2007benchmark}, USPS \cite{zeng2014image}, C-Cube \cite{Camastra2006Offline,DBLP:journals/pr/FeiXFY17}, PEI and Extended Yale B database \cite{wright2009robust}. Our method is compared with the traditional Kmeans, SSC \cite{elhamifar2009sparse}, LRR \cite{liu2013robust}, LSR \cite{lu2012robust}, CASS \cite{lu2013correlation}, MoG Regression \cite{li2015subspace} and NSSC \cite{DBLP:journals/jmlr/WangX16}. SSC, LRR, LSR, CASS, MoG Regression and NSSC are representative subspace clustering methods which are introduced in section \ref{section 2}. For fair comparison with the previous methods, we adopt the same preprocessing for the whole databases: use PCA to reduce the dimension of the original data and keep nearly 98 percent energy. Besides, the parameters of each method are manually tuned to achieve their best performance. Finally, we employ the clustering accuracy (AC) \cite{elhamifar2013sparse,cao2015diversity} and the normalized mutual information metric (NMI) \cite{strehl2002cluster,cai2005document} to evaluate the subspace clustering results. From the experimental results, we can see that our method achieves better performance than other state-of-the-art methods.
\subsection{Data sets \label{ExperimentA}}
We firstly give the detailed description about five real data sets used in the experiments.
\begin{itemize}
  \item The first data set is the Hopkins 155 motion segmentation database. It consists of 155 video sequences, where 120 of the videos have two motions and 35 of the videos contain three motions (a motion corresponding to a subspace). For each video, feature trajectories have been extracted for clustering. The number of feature trajectories of each video ranges from 39 to 550. Each video can be regarded as a subspace clustering task, and so there are 155 subspace segmentation tasks totally.
  \item The second is the USPS database which is one of the standard data sets for handwritten digit recognition \cite{DBLP:conf/cvpr/YangPB16}. It contains 9298 images of hand-written digits from 0 to 9. The size of each image is $16\times16$. To reduce the memory consumption in our experiments, we randomly select 30 images for each digit to construct a subset with 300 samples.
  \item The third is the C-Cube cursive character data set which contains both the upper and lower case of 26 letters. It has 57646 character images and the average dimension of all images is about 3120. For each subject, we randomly select 20 images to form a subset for our experiments. Then each image is normalized to $24\times24$ pixel array and reshaped to a vector.
  \item The forth data set is the FEI part 1 database. This database is the subset of the whole FEI database. It contains 700 images with 50 subjects, and each subject has 14 images captured from a large range of views.
  \item The fifth data set is the Extended Yale B Database which is a popular dataset for image clustering \cite{7434007,DBLP:journals/tip/ChengYYFH10,7517387}. It consists of 2414 frontal face images of 38 subjects, and each subject has about 64 frontal face images with different pose, angle and illumination conditions. In our experiment, we construct three subspace clustering tasks based on the first 5, 8 and 10 subjects, and each subject has 64 face images.

\end{itemize}
\begin{figure}[!t]
\centering
\subfigure[Hopkins 155 motion segmentation database]{\includegraphics[width=7.0cm]{{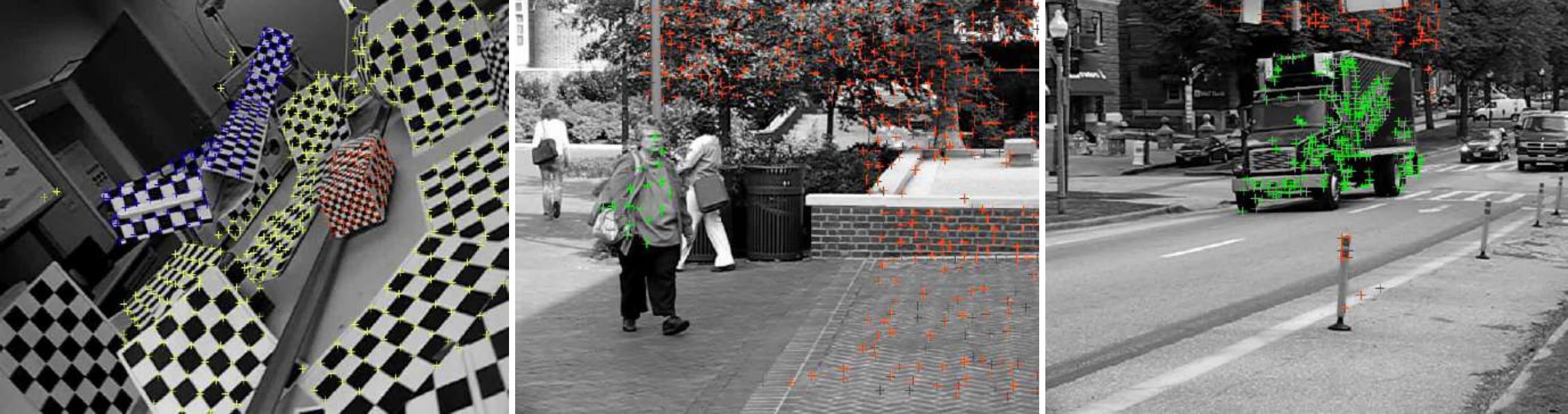}}%
\label{fig_first_case}}
\hfil
\subfigure[USPS database]{\includegraphics[width=7cm]{{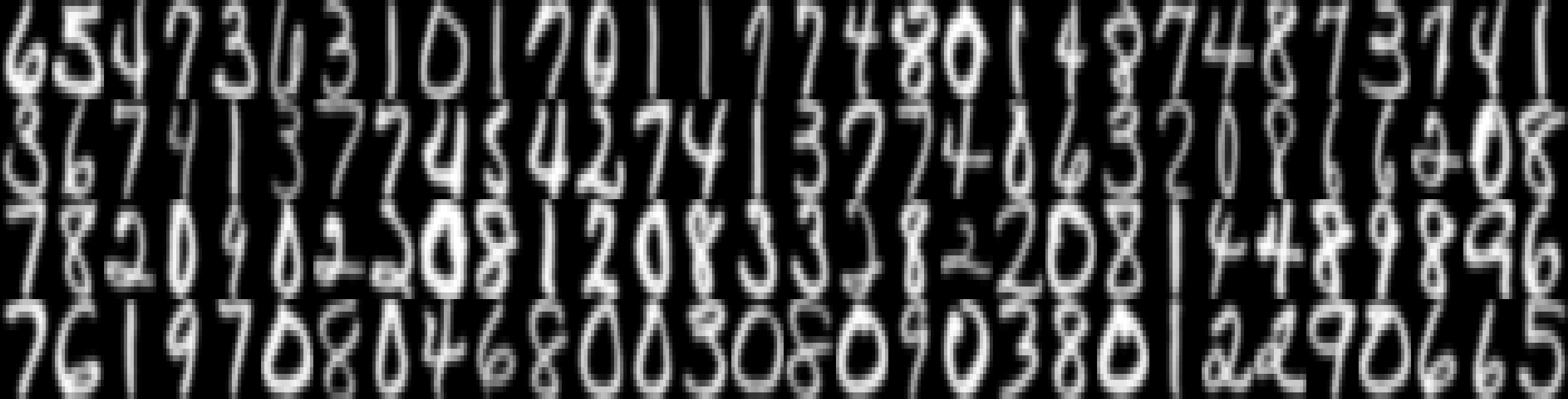}}%
\label{fig_second_case}}
\hfil
\subfigure[C-Cube database]{\includegraphics[width=7cm]{{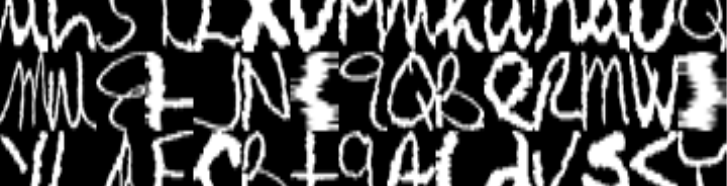}}%
\label{fig_third_case}}
\hfil
\subfigure[FEI database]{\includegraphics[width=7cm]{{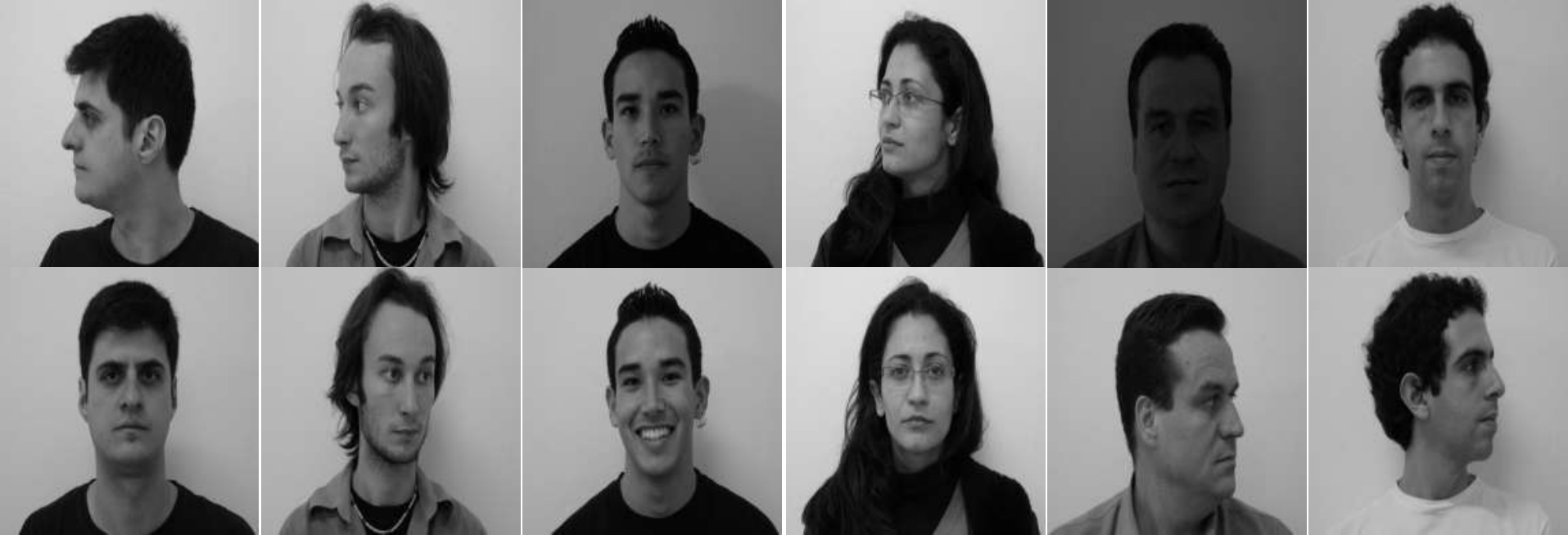}}%
\label{fig_forth_case}}
\hfil
\subfigure[Extended Yale B database]{\includegraphics[width=7cm]{{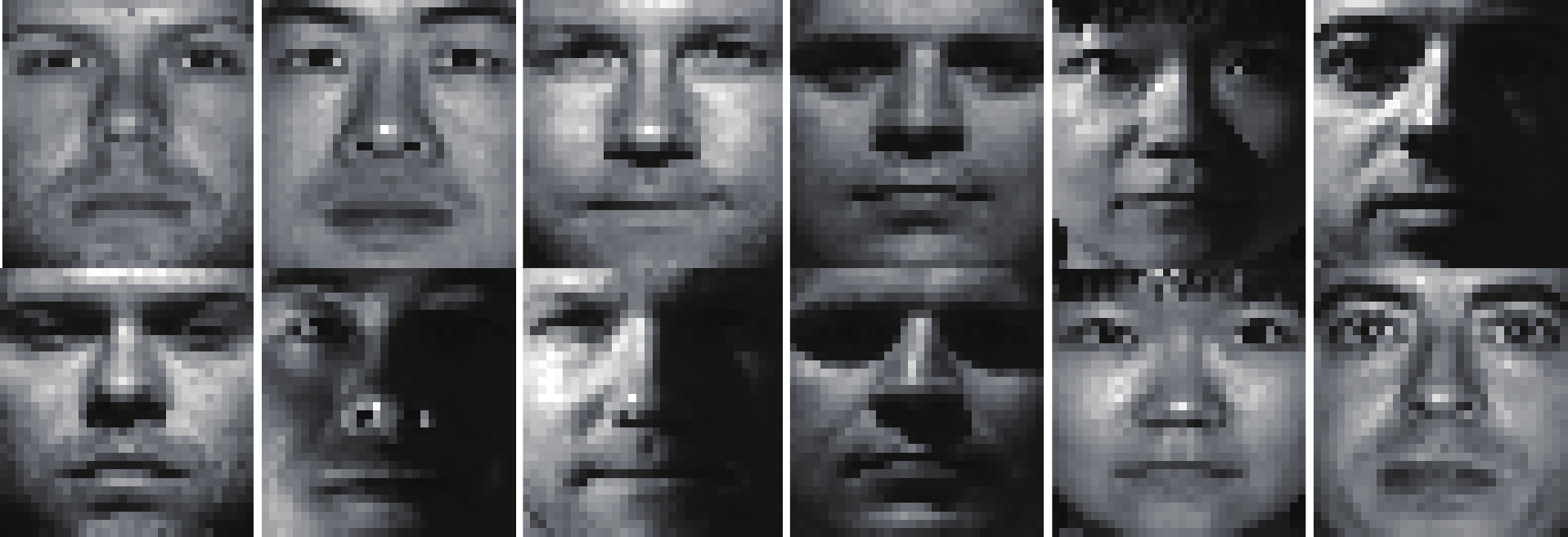}}%
\label{fig_fifth_case}}
\caption{Examples of different data sets. For the Hopkins 155 database, we simply choose some frames in the videos. The motion objects in these three frames are checkerboard, people and truck, respectively. For the FEI  and Extended Yale B database, each column represents a single subject.}
\label{Dataset}
\end{figure}

Fig. \ref{Dataset} gives some samples of these five data sets. From Fig. \ref{fig_fifth_case}, we can see that Extended Yale B is a tough database for subspace clustering due to its large noise. So we can further verify the effectiveness of our method in handling the noise. Table \ref{TDataset} gives the statistics of these databases. For the Hopkins 155 database, the values of size and dimensionality represent the average of the whole videos, and the class of each video is 2 or 3.
\begin{table}[t]
\renewcommand{\arraystretch}{1}
\caption{Statistics of five data sets.}\label{TDataset}
\begin{center}
\begin{tabular}{|c||c|c|c|}
\hline
dataset &size & dimensionality & $\#$ of classes \\
\hline
 Hopkins 155 & 59 & 296 & 2 or 3\\
\hline
USPS & 9298 & 256 & 10\\
\hline
C-Cube & 57646 & 3120 & 52\\
\hline
FEI & 700 & 768 & 50\\
\hline
Extended Yale B & 2414 & 1024 & 68\\
\hline
\end{tabular}
\end{center}
\end{table}

\begin{figure*}[!t]
\centering
\subfigure[Hopkins 155]{\includegraphics[width=3.49cm]{{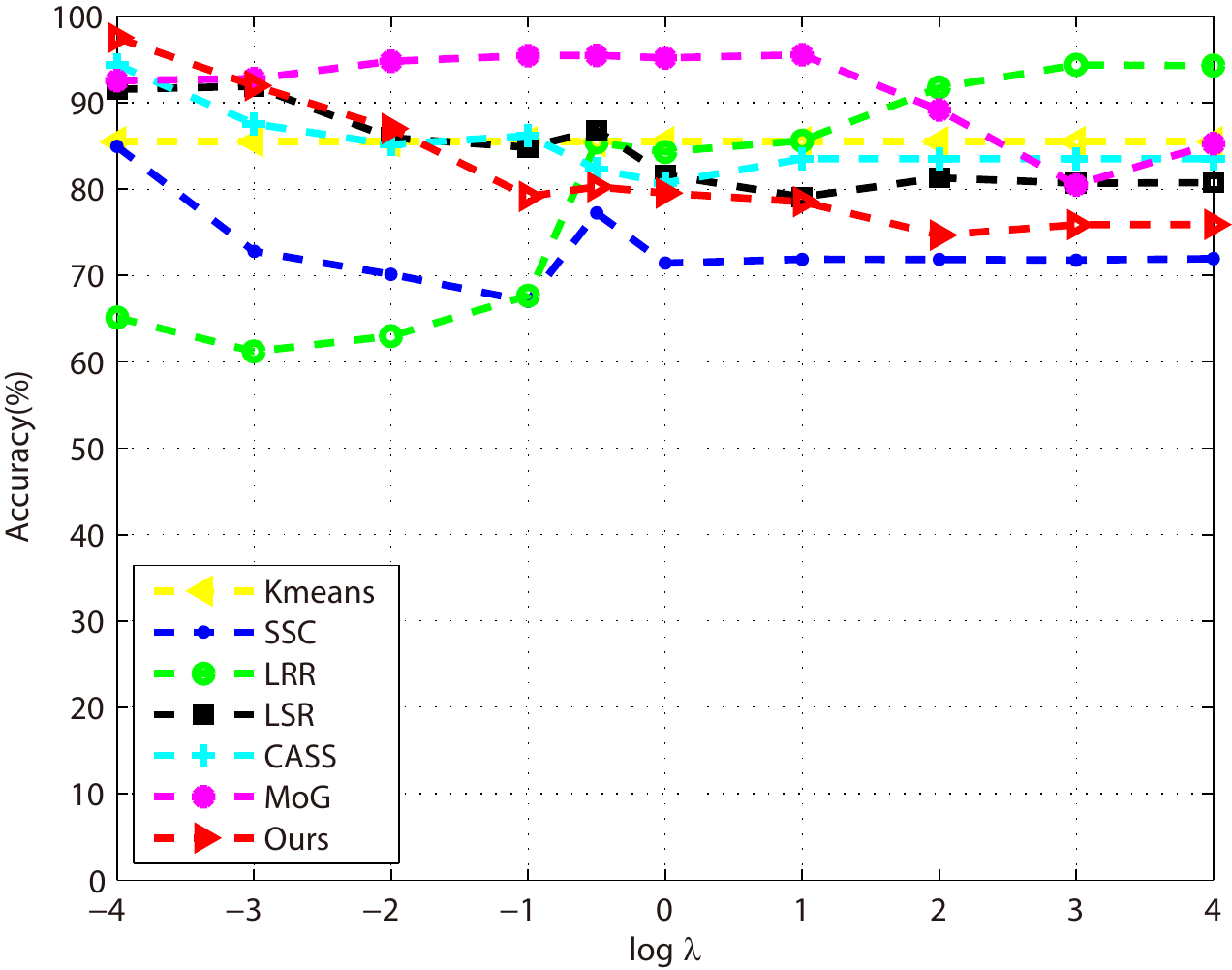}}%
\label{Para_lambda_1}}
\subfigure[USPS]{\includegraphics[width=3.49cm]{{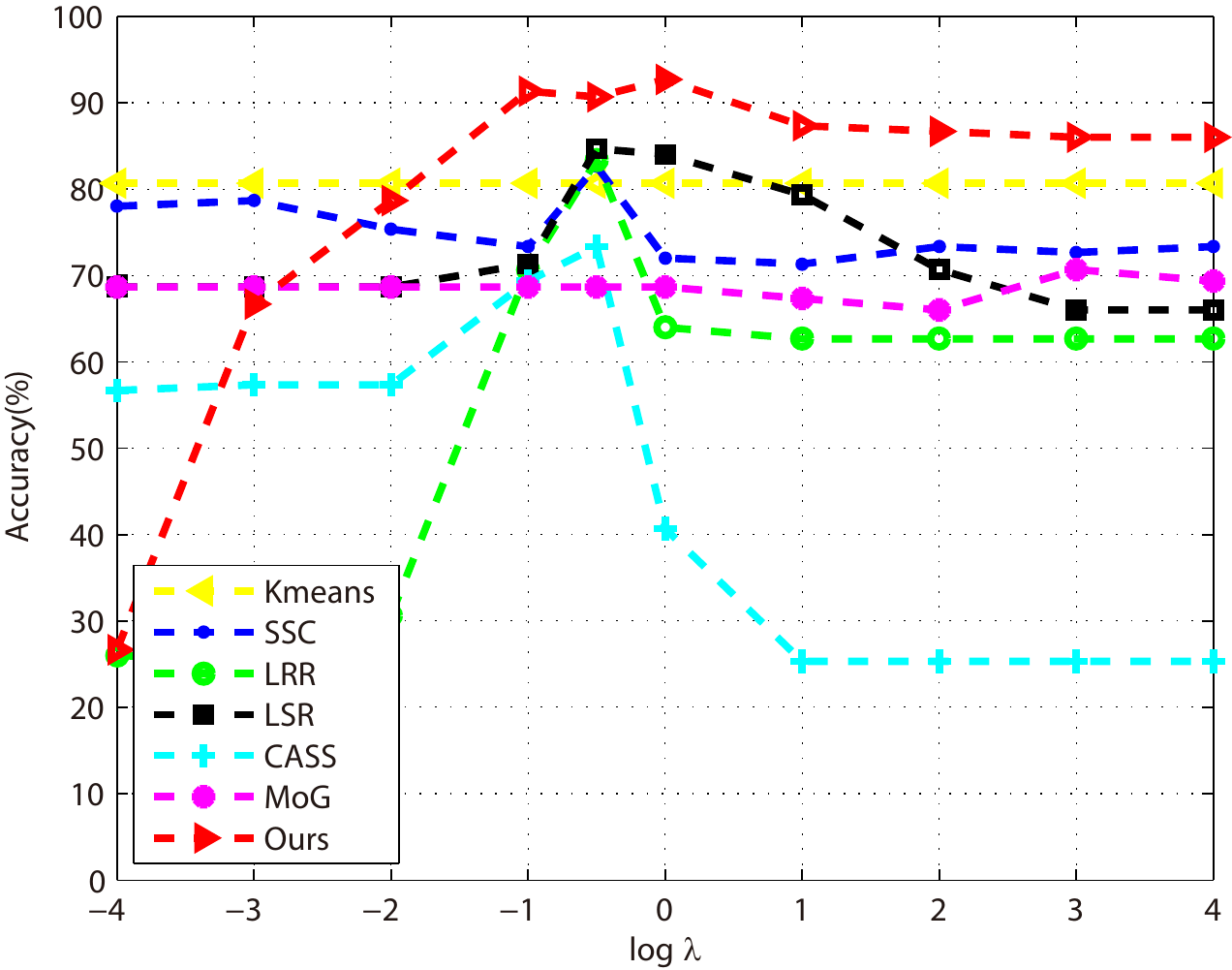}}%
\label{Para_lambda_2}}
\subfigure[C-Cube]{\includegraphics[width=3.49cm]{{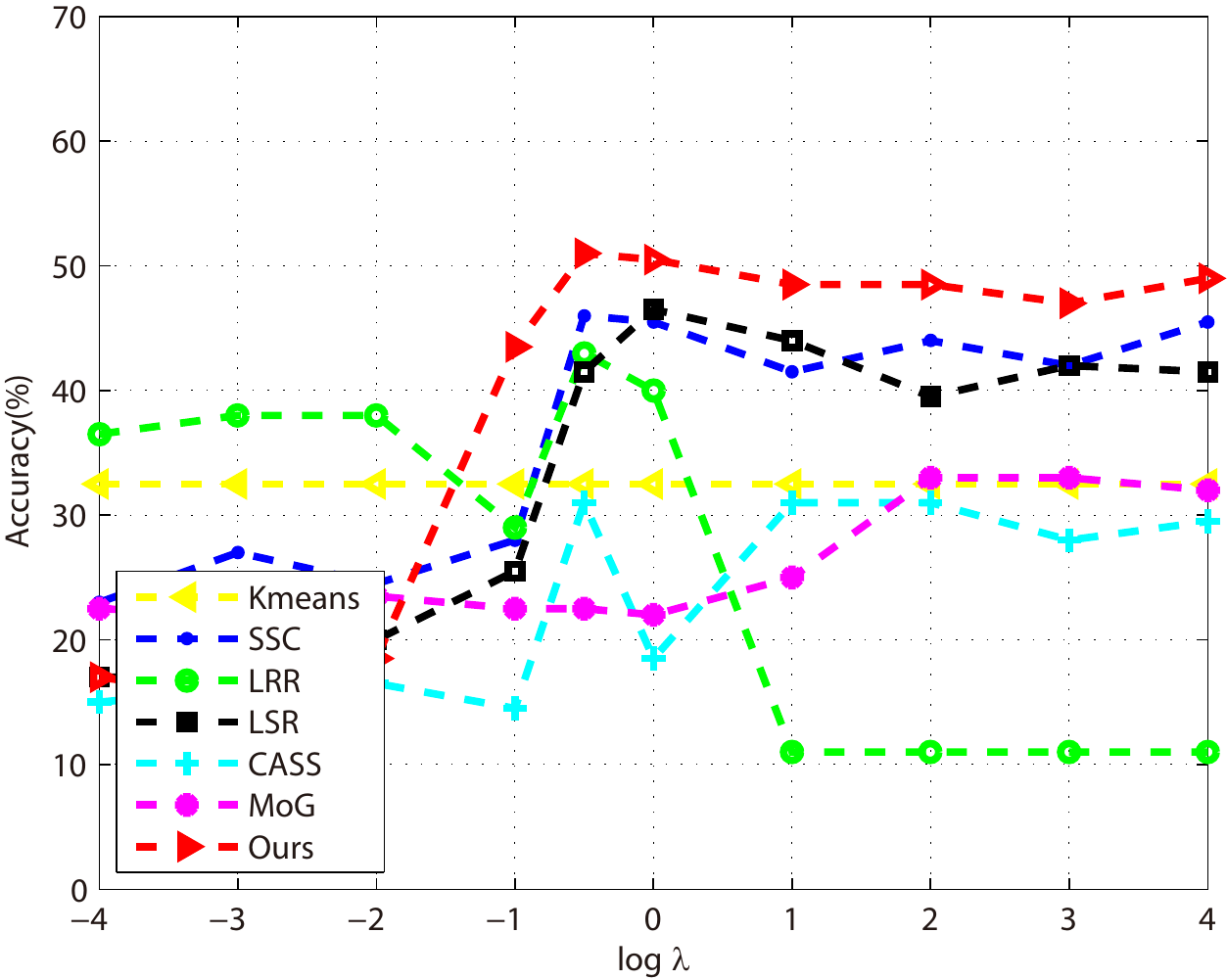}}%
\label{Para_lambda_3}}
\subfigure[FEI]{\includegraphics[width=3.49cm]{{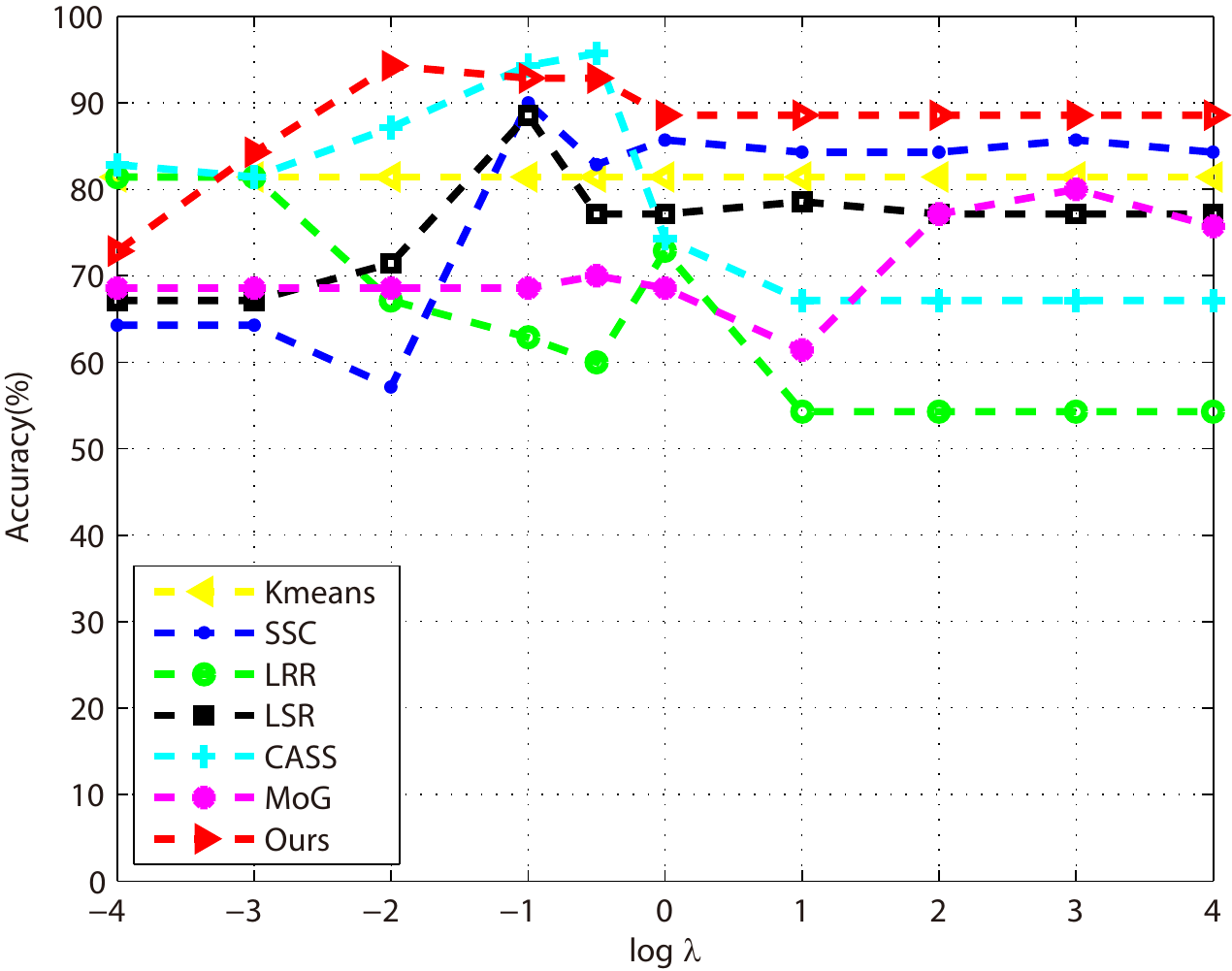}}
\label{Para_lambda_4}}
\subfigure[Extended Yale B]{\includegraphics[width=3.49cm]{{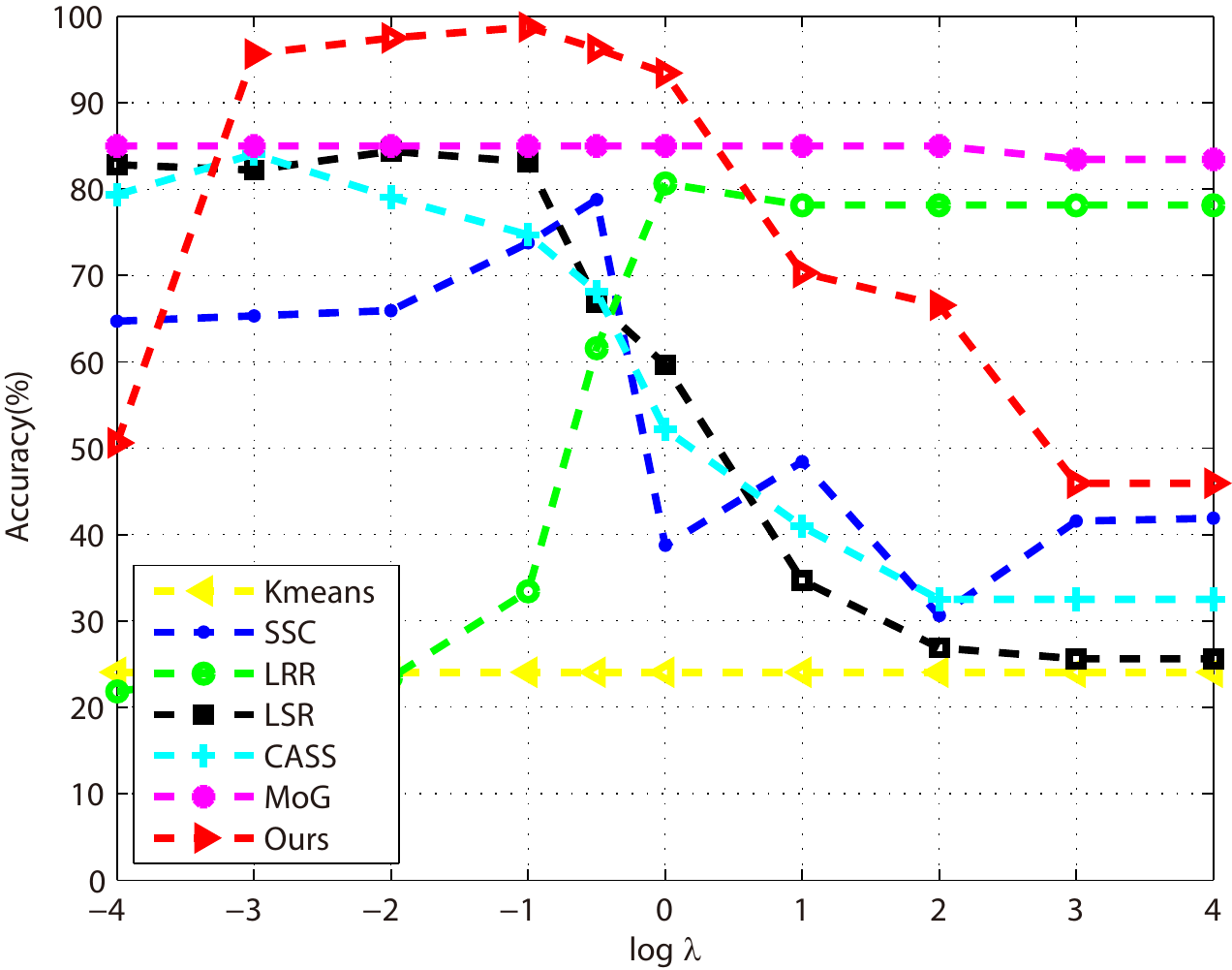}}
\label{Para_lambda_5}}

\caption{The performance of different methods versus parameter $\lambda$.}
\label{Para_lambda}
\end{figure*}

\begin{figure*}[!t]
\centering
\subfigure[Hopkins 155]{\includegraphics[width=3.49cm]{{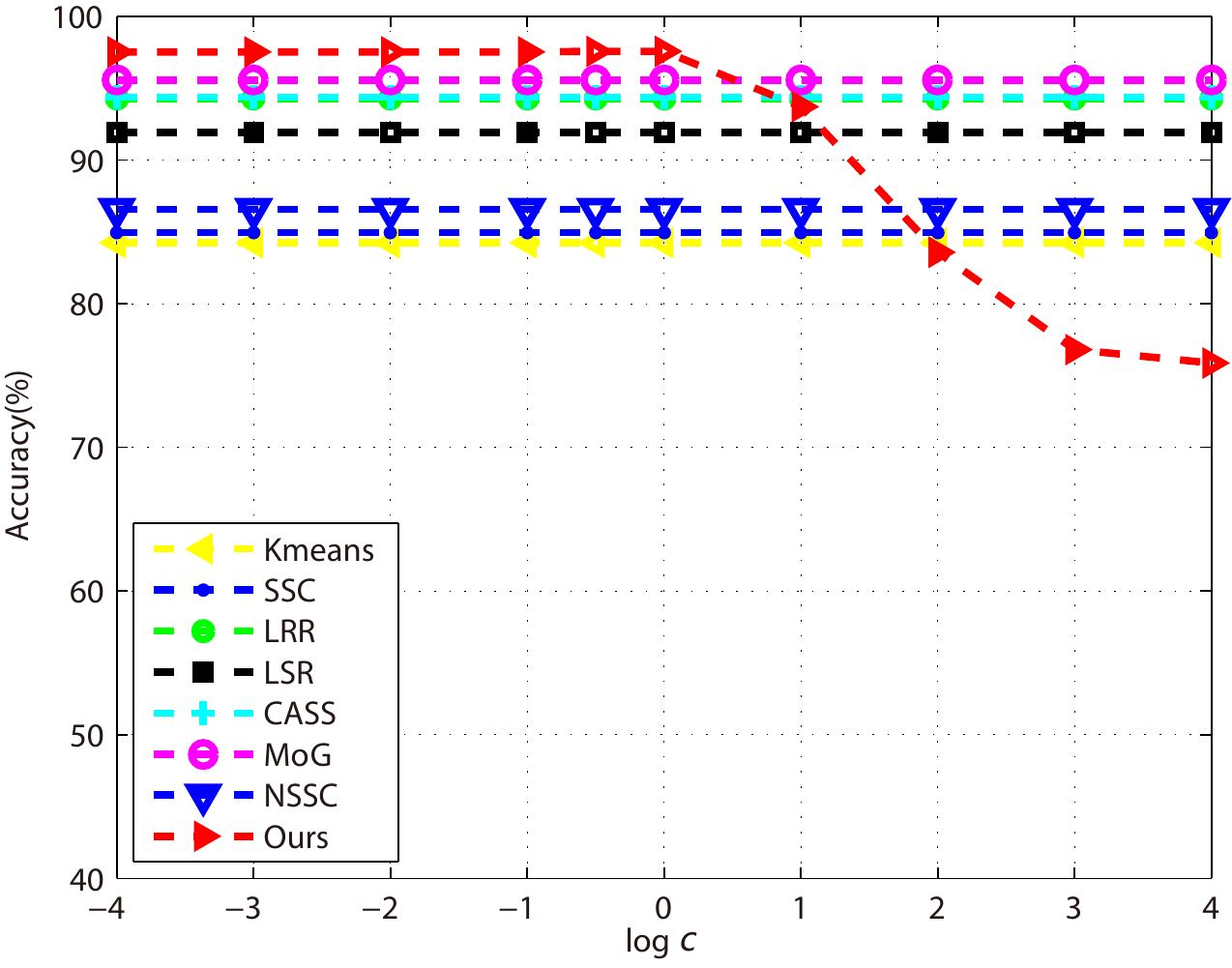}}%
\label{Para_c_1}}
\subfigure[USPS]{\includegraphics[width=3.49cm]{{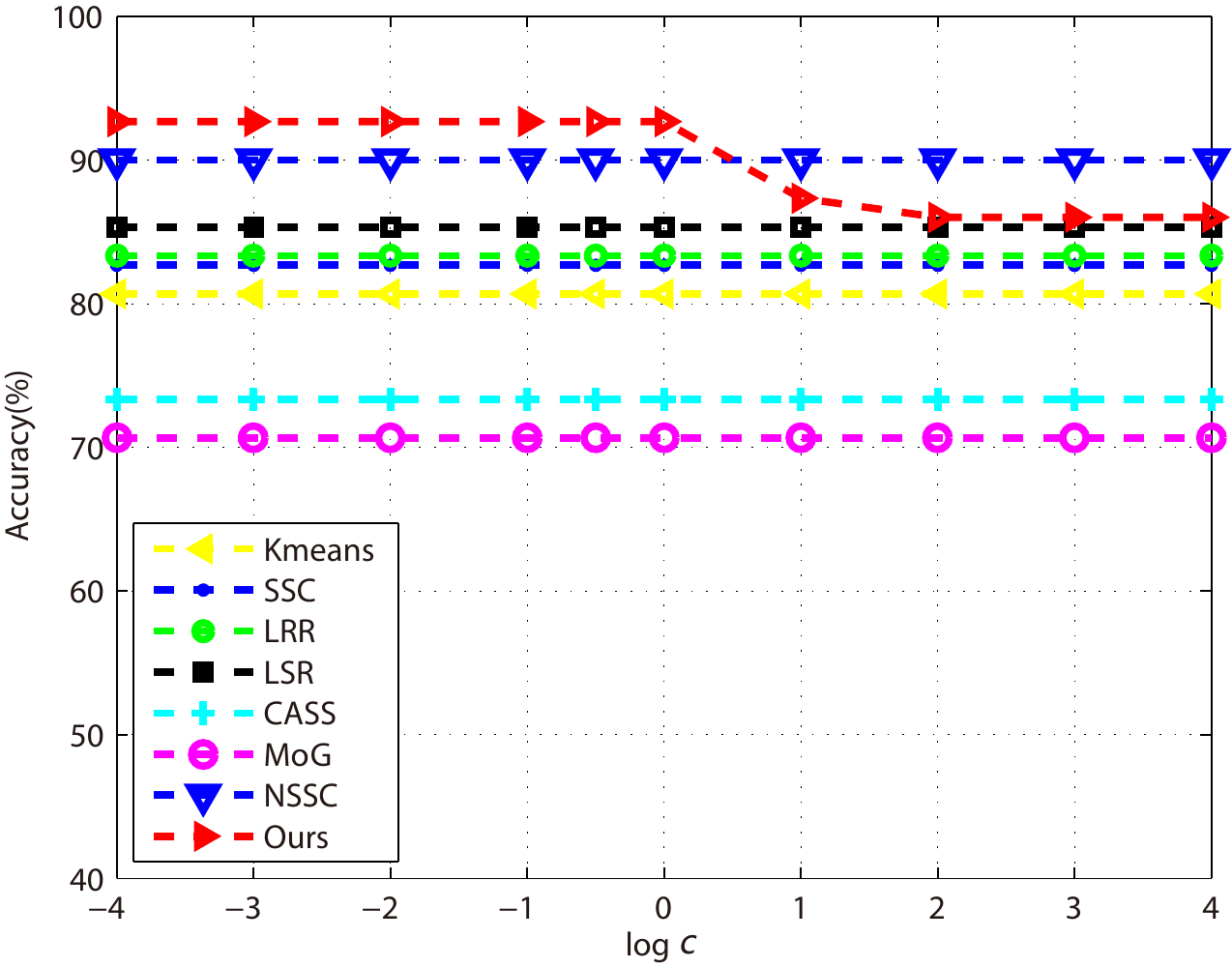}}%
\label{Para_c_2}}
\subfigure[C-Cube]{\includegraphics[width=3.49cm]{{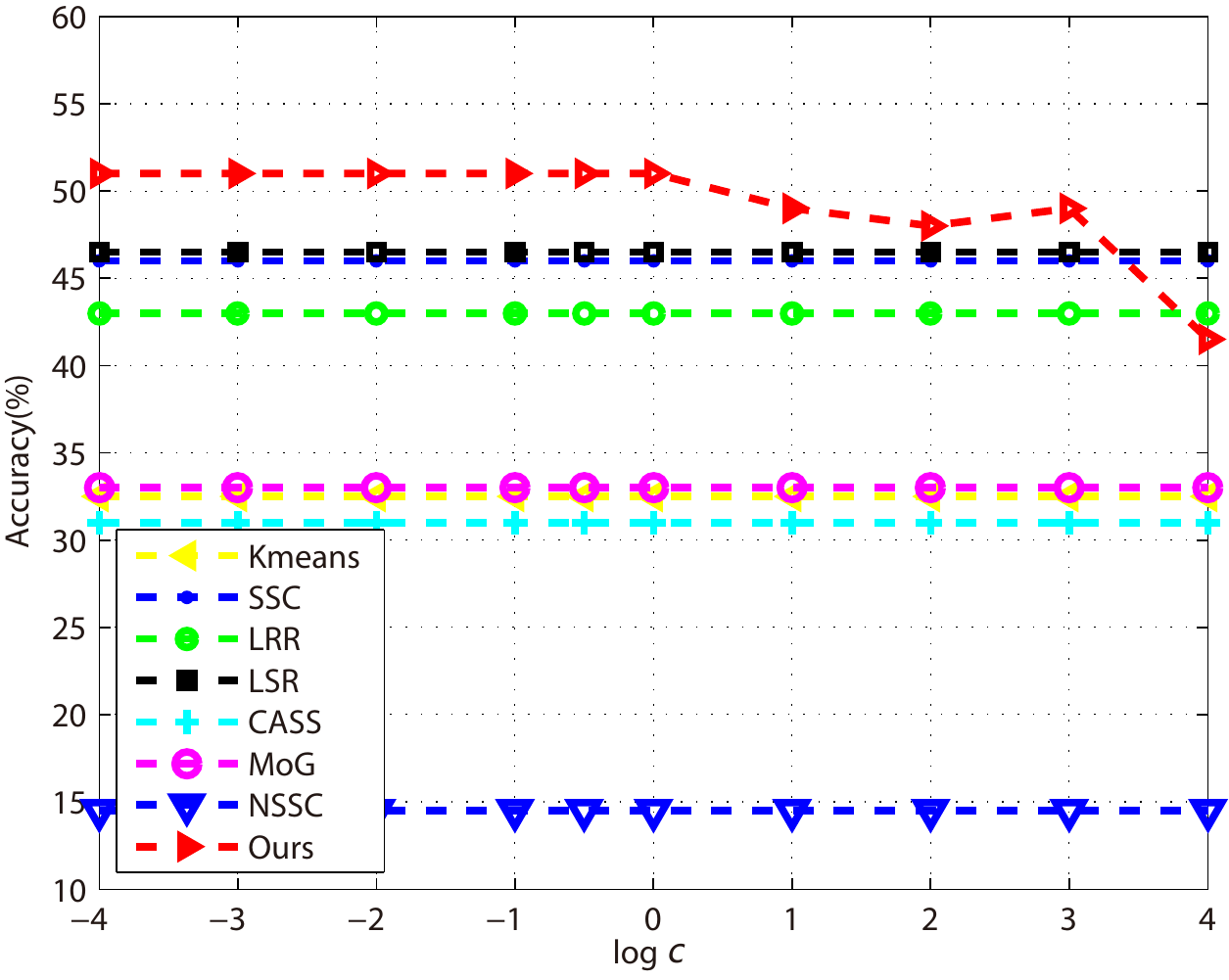}}%
\label{Para_c_3}}
\subfigure[FEI]{\includegraphics[width=3.49cm]{{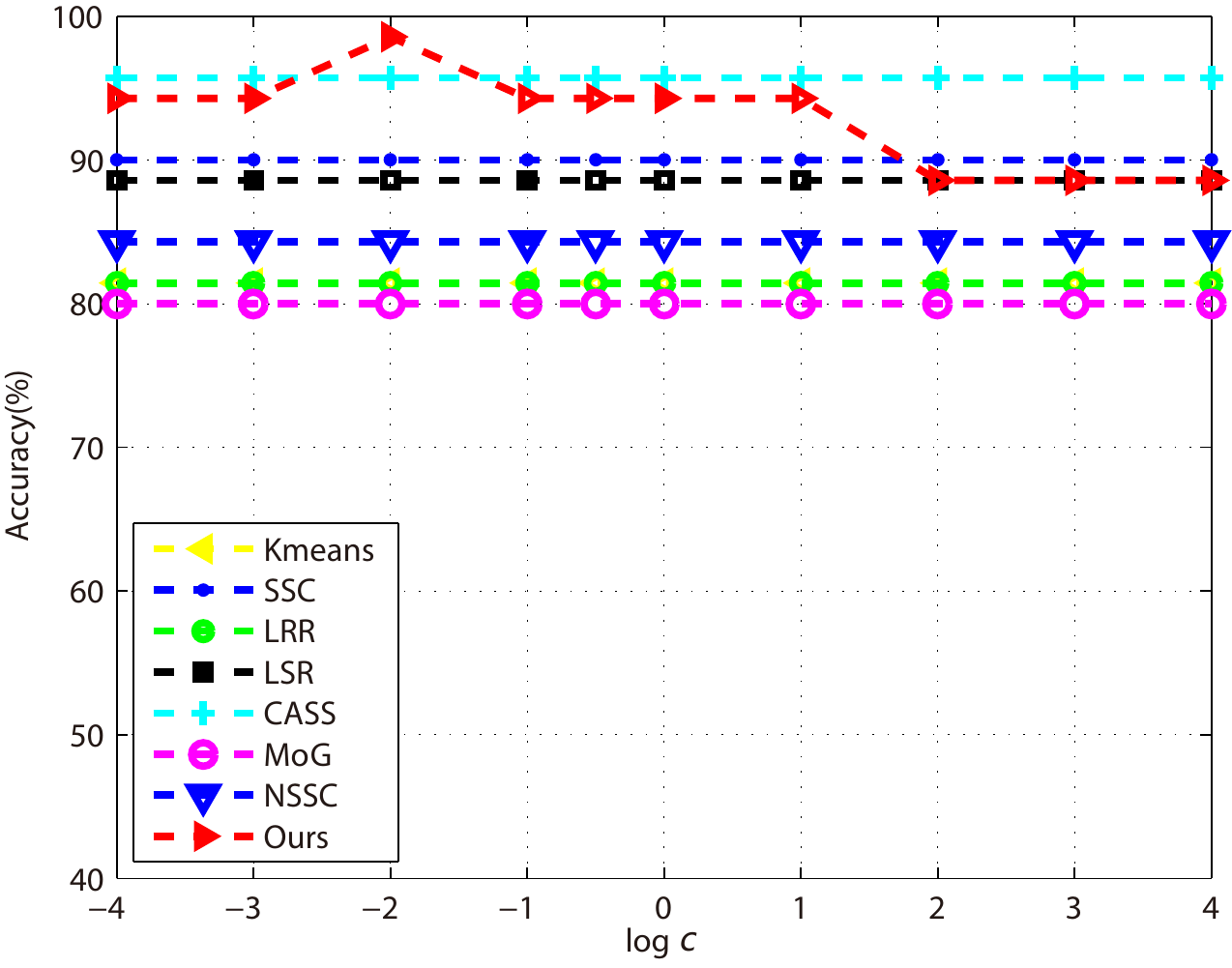}}
\label{Para_c_4}}
\subfigure[Extended Yale B]{\includegraphics[width=3.49cm]{{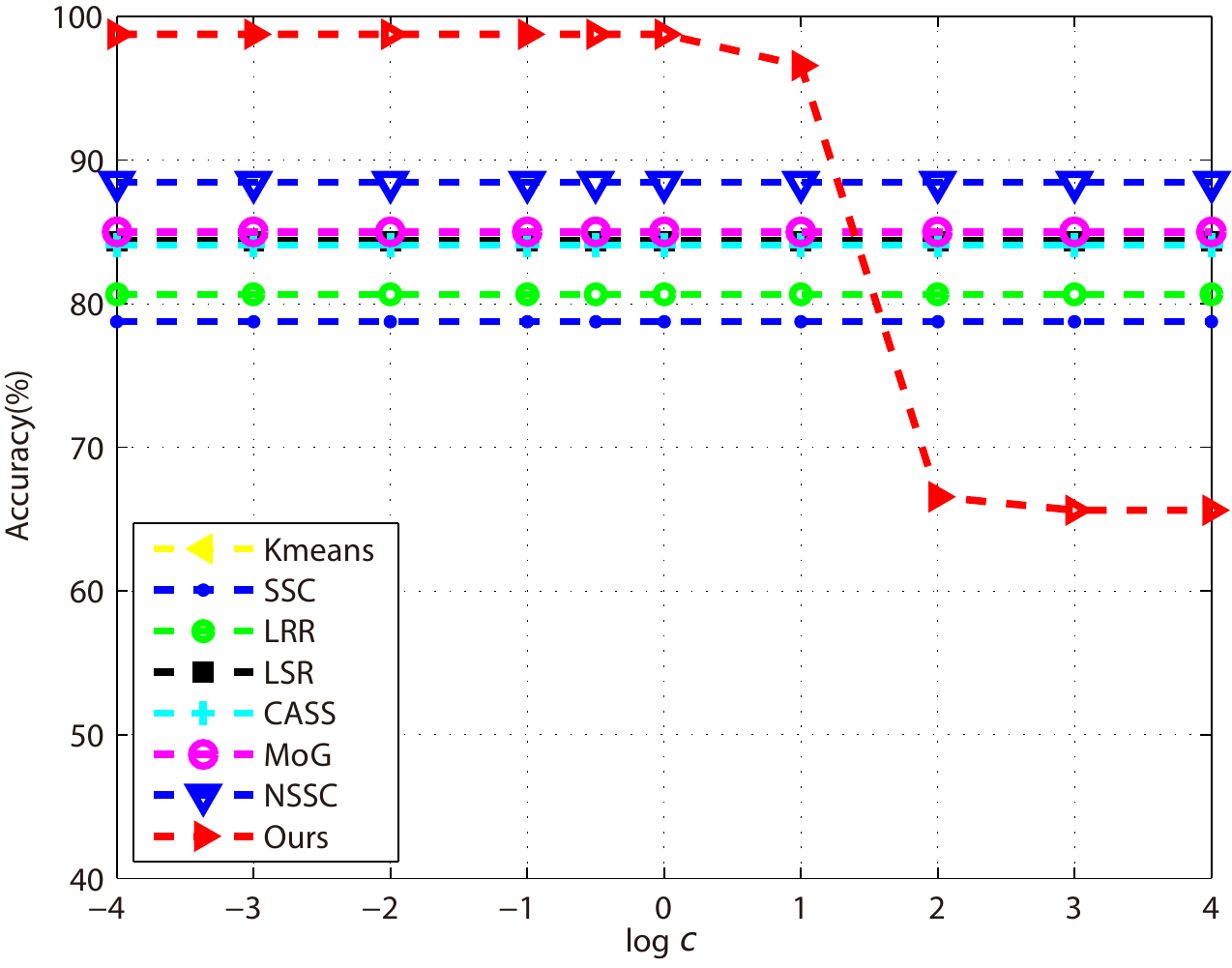}}
\label{Para_c_5}}
\caption{The performance of different methods versus parameter $c$.}
\label{Para_c}
\end{figure*}

\subsection{Evaluation Criterion}
The clustering results are evaluated by comparing the obtained label of each subspace clustering method with the groundtruth. The clustering accuracy (AC) and the normalized mutual information (NMI), as two popular metrics, are employed to measure the clustering performance.
\begin{table*}[t]
\renewcommand{\arraystretch}{1}
\caption{The best parameter for each method on different databases.}\label{Tab_para}
\begin{center}
\begin{tabular}{|c||c|c|c|c|c|c|c|}
\hline
dataset & SSC ($\lambda$) & LRR ($\lambda$) & LSR ($\lambda$) & CASS ($\lambda$) & MoG ($\lambda$) & NSSC ($\lambda$) & Ours ($\lambda$, $c$) \\
\hline
 Hopkins 155    & 0.0001 & 1000 & 0.001 & 0.0001 & 10 & 10  & (0.0001, 0.5)\\
\hline
USPS            & 0.5    & 0.5  & 0.5   & 0.5    & 1000 & 10 & (1,0.1)\\
\hline
C-Cube            & 0.5    & 0.5  & 1   & 0.5    & 1000 & 10 & (0.5,0.1)\\
\hline
FEI             & 0.1    & 0.001 & 0.1  & 0.5    & 1000 & 10 & (0.01,0.01)\\
\hline
Extended Yale B & 0.5    & 1     & 0.01 & 0.001  & 10   & 10 & (0.1,0.1) \\
\hline
\end{tabular}
\end{center}
\end{table*}

Given an obtained label vector ${\bf{o}}_i$ and a corresponding groundtruth label vector ${\bf{g}}_i$. The AC is calculated by

\begin{equation}\label{39}
AC = \frac{{\sum\limits_{j = 1}^n {\delta ({{\bf{g}}_i}(j),{{\bf{o}}_i}^\prime (j))} }}{n},
\end{equation}

\begin{equation}\label{40}
\delta (x,y) = \left\{ \begin{array}{l}
1,{\kern 1pt} {\kern 1pt} {\kern 1pt} {\kern 1pt} {\kern 1pt} {\kern 1pt} {\kern 1pt} {\kern 1pt} {\kern 1pt} {\kern 1pt} if{\kern 1pt} {\kern 1pt} x = y\\
0,{\kern 1pt} {\kern 1pt} {\kern 1pt} {\kern 1pt} {\kern 1pt} {\kern 1pt} {\kern 1pt} {\kern 1pt} else
\end{array} \right.,
\end{equation}
where ${{\bf{o}}_i}^\prime  = map({{\bf{o}}_i})$. $map({\bf{o}}_i)$ is the permutation mapping function that chooses ${\bf{g}}_i$ as a reference vector and maps each element in ${\bf{o}}_i$ to the equivalent label in ${\bf{g}}_i$. So $map({\bf{o}}_i)$ is designed for solving the problem of correspondence between two label vectors. Kuhn-Munkres algorithm can be utilized to find the best mapping.

Mutual Information (MI), as a symmetric measure to quantify the information shared between two statistical distributions, provides a degree of agreement between two clustering results. Let $c_p$ be the cluster obtained from the groundtruth ${\bf{g}}_i$ and $c'_q$ obtained from our clustering result ${\bf{o}}_i$. Then the corresponding MI is defined as follow:
\begin{equation}\label{41}
MI({\bf{g}}_i,{\bf{o}}_i) = \sum\limits_{p = 1}^k {\sum\limits_{q = 1}^{{k'}} {\frac{{{n_{pq}}}}{n}\log \left( {\frac{{\frac{{{n_{pq}}}}{n}}}{{\frac{{{n_p}}}{n} \cdot \frac{{n'_q}}{n}}}} \right)} ,}
\end{equation}
where $k$ and $k'$ denote the number of clusters in groundtruth and our clustering result, respectively. $n_{p}$ is the number of points in cluster $c_p$, $n'_q$ is the number of points in cluster $c'_q$ and $n_{pq}$ denotes the number of shared points between $c_p$ and $c'_q$. In order to obtain a normalized version of MI that ranges from 0 to 1, we use the NMI metric as
\begin{equation}\label{42}
NMI({{\bf{g}}_i},{{\bf{o}}_i}) = \frac{{2 \cdot MI({{\bf{g}}_i},{{\bf{o}}_i})}}{{H({{\bf{g}}_i}) + H({{\bf{o}}_i})}},
\end{equation}
where $H(\cdot)$ denotes the entropy function.

\begin{table*}[!t]
\renewcommand{\arraystretch}{1}
\caption{The clustering results of different algorithms on the Hopkins 155 database. The best results are in bold font.}\label{Tab_Hopkins}
\begin{center}
\begin{tabular}{|c c|C{1cm}|C{1cm}|C{1cm}|C{1cm}|C{1cm}|C{1cm}|C{1cm}|C{1cm}|}
\hline
\multirow{2}{*}{k}
  & &\multicolumn{8}{c|}{Accuracy}\\
    \cline{3-10}
& & Kmeans & SSC & LRR & LSR & CASS & MoG & NSSC & Ours \\
\hline\hline
2 motions &Average & 87.80 & 83.40 & 96.47 & 96.14 & 92.01 & \textbf{98.03} & 88.76 & 97.81 \\
Acc.(\%) & Median & 88.10 & 83.83 & 99.67 & 99.54 & 99.64 & \textbf{100.00} & 90.23 & \textbf{100.00} \\
\hline
3 motions &Average & 77.22 & 74.88 & 90.38 & 90.66 & 89.67 & 94.25 & 78.46 &  \textbf{95.03}\\
Acc.(\%) & Median & 80.42 & 75.45 & 94.57 & 92.34 & 91.43 & 97.66 & 79.90 & \textbf{99.17} \\
\hline
Total &Average & 85.55 & 81.48 & 95.08 & 94.96 & 91.55 & \textbf{97.21} & 86.37 & \textbf{97.21} \\
Acc.(\%) & Median & 85.86 & 80.84 & 99.41 & 99.06 & 97.76 & 99.71 & 88.50& \textbf{100.00} \\
\hline
\multirow{2}{*}{k}
  & &\multicolumn{8}{c|}{Normalized Mutual Information}\\
    \cline{3-10}
& & Kmeans & SSC & LRR & LSR & CASS & MoG & NSSC & Ours \\
\hline\hline
2 motions &Average & 53.96 & 40.09 & 86.53 & 79.60 & 70.42 & 86.10 & 57.37 & \textbf{87.24} \\
Acc.(\%) & Median & 44.11 & 27.96 & 96.43 & 94.92 & 96.19 & \textbf{100.00} & 57.88 & \textbf{100.00} \\
\hline
3 motions &Average & 49.69 & 43.33 & 80.19 & 76.01 & 77.67 & 83.21 & 50.22 & \textbf{86.61}\\
Acc.(\%) & Median & 47.93 & 46.90 & 80.14 & 76.86 & 79.47 & 89.17 & 47.34 & \textbf{95.41} \\
\hline
Total &Average & 53.26 & 40.96 & 85.17 & 78.85 & 72.14 & 85.50 & 55.78 & \textbf{87.12} \\
Acc.(\%) & Median & 45.14 & 34.65 & 94.42 & 92.19 & 85.59 & 96.96 & 56.23 & \textbf{100.00} \\
\hline
\end{tabular}
\end{center}
\end{table*}
\begin{table*}[!t]
\renewcommand{\arraystretch}{1}
\caption{The clustering results of different algorithms on the USPS database. The best results are in bold font.}\label{Tab_USPS}
\begin{center}
\begin{tabular}{|c|C{1cm}|C{1cm}|C{1cm}|C{1cm}|C{1cm}|C{1cm}|C{1cm}|C{1cm}|}
\hline
\multirow{2}{*}{k}
  & \multicolumn{8}{c|}{Accuracy}\\
    \cline{2-9}
  & Kmeans & SSC & LRR & LSR & CASS & MoG & NSSC & Ours \\
\hline
\hline
5 subjects & 80.67 & 82.67 & 83.33 & 85.33 & 73.33 & 70.66 & 90.00 & \textbf{92.67}\\
\hline
6 subjects & 75.00 & 82.77 & 83.89 & 80.00 & 70.00 & 62.22 & 81.67 & \textbf{87.78}\\
\hline
7 subjects & 77.14 & 80.00 & 75.24 & 80.95 & 73.81 & 58.10 & 81.90 & \textbf{83.33}\\
\hline
8 subjects & 78.75 & 79.85 & 76.25 & 79.17 & 71.25 & 54.17 & 82.08 & \textbf{86.25}\\
\hline
9 subjects & 77.78 & 80.00 & 69.63 & 80.74 & 75.56 & 55.93 & 80.00 & \textbf{85.56}\\
\hline
10 subjects & 73.00 & 67.67 & 70.00 & 76.33 & 71.00 & 55.67 & 77.00 &\textbf{81.33}\\
\hline
\multirow{2}{*}{k}
  & \multicolumn{8}{c|}{ Normalized Mutual Information}\\
    \cline{2-9}
  & Kmeans & SSC & LRR & LSR & CASS & MoG & NSSC & Ours \\
\hline
\hline
5 subjects & 66.10 & 71.47 & 72.57 & 69.00 & 66.76 & 45.52 & 76.86 & \textbf{82.86}\\
\hline
6 subjects & 60.69 & 63.86 & 73.84 & 65.74 & 62.42 & 46.37 & 70.64 & \textbf{77.56}\\
\hline
7 subjects & 64.45 & 64.88 & 64.49 & 68.90 & 63.92 & 42.24 & 74.02 & \textbf{74.63}\\
\hline
8 subjects & 68.48 & 72.22 & 66.69 & 70.29 & 64.03 & 42.13 & 74.96 & \textbf{80.07}\\
\hline
9 subjects & 67.28 & 68.63 & 67.37 & 71.49 & 66.23 & 46.87 & 74.31 & \textbf{78.87}\\
\hline
10 subjects & 63.26 & 59.93 & 63.95 & 68.87 & 62.48 & 45.54 & 69.93 &\textbf{74.86}\\
\hline
\end{tabular}
\end{center}
\end{table*}

\subsection{Parameter Selection}
Our proposed method has two essential parameters: the weight factor $\lambda$ and a constant $c$. Then we conduct the corresponding comparison experiments to choose the best parameter for each method on the whole databases. To reduce the memory consumption in our experiments, we only use the first five videos of the Hopkins 155 database to choose the appropriate parameters. For the USPS, NSSC, FEI and Extended Yale B databases, we use the first five subjects to select the parameters, respectively. Besides, we set the range of $\lambda$ and $c$ as $[10^{-4},10^4]$.

Fig. \ref{Para_lambda} gives the performance of different methods with the parameter $\lambda$.
For NSSC, when $\lambda < 1$, its optimization method usually fails to give a local optimal solution, and it can give a good performance when $\lambda = 10$. So we fix $\lambda = 10$ for NSSC on the whole datasets.
Note that our method can give a best performance when $\lambda=10^{-4}$ on Hopkins 155 database. Hence, we fix $\lambda=10^{-4}$ for our method on the Hopkins 155. For the USPS database, our method obtains a better performance than other methods when $\lambda$ is larger than 0.01 and gives the largest CI when $\lambda = 1$. For C-Cube, our proposed method gives the best clustering result when $\lambda = 0.5$. For FEI and Extended Yale B databases, our method shows its effectiveness when $\lambda$ is around 0.01. Compared with other methods, MoG can give a stable performance on these five databases with respect to the parameter $\lambda$ while it gives a bad clustering accuracy on the USPS and FEI databases. For the USPS and Extended Yale B databases, the curve of LRR and CASS both give a bigger fluctuation. Because Kmeans has no parameter, its accuracy curve is a straight line. Note that the Kmeans algorithm gives a very low performance on the Extended Yale B database.

\begin{table*}[!t]
\renewcommand{\arraystretch}{1}
\caption{The clustering results of different algorithms on the C-Cube database. The best results are in bold font.}\label{Tab_CCC}
\begin{center}
\begin{tabular}{|c|C{1cm}|C{1cm}|C{1cm}|C{1cm}|C{1cm}|C{1cm}|C{1cm}|C{1cm}|}
\hline
\multirow{2}{*}{k}
  & \multicolumn{8}{c|}{Accuracy}\\
    \cline{2-9}
  & Kmeans & SSC & LRR & LSR & CASS & MoG & NSSC & Ours \\
\hline
\hline
10 subjects & 32.50 & 46.50 & 43.00 & 45.50 & 22.50 & 33.00 & 14.50 & \textbf{51.00}  \\
\hline
20 subjects & 26.25 & \textbf{44.00} & 32.25 & 33.25 & 27.75 & 26.25 & 21.00 & 37.50  \\
\hline
30 subjects & 24.33 & 32.50 & 29.67 & 34.33 & 27.67 & 24.67 & 28.83 & \textbf{35.17} \\
\hline
40 subjects & 22.87 & 30.50 & 28.00 & 28.75 & 25.86 & 24.50 & 28.62 & \textbf{32.37} \\
\hline
50 subjects & 23.70 & 29.50 & 28.10 & 32.30 & 26.60 & 26.70 & 25.70 & \textbf{32.40} \\
\hline
\multirow{2}{*}{k}
  & \multicolumn{8}{c|}{ Normalized Mutual Information}\\
    \cline{2-9}
  & Kmeans & SSC & LRR & LSR & CASS & MoG & NSSC & Ours \\
\hline
\hline
10 subjects & 30.22 & 36.86 & 37.56 & 43.70 & 17.26 & 28.32 & 10.45 & \textbf{46.61} \\
\hline
20 subjects & 35.16 & 44.36 & 42.97 & 41.86 & 36.53 & 29.04 & 29.28 & \textbf{45.19} \\
\hline
30 subjects & 37.96 & 40.46 & 44.96 & 45.72 & 41.31 & 35.24 & 42.67 & \textbf{48.51} \\
\hline
40 subjects & 40.77 & 40.50 & 47.27 & 47.26 & 43.77 & 40.87 & 45.19 & \textbf{48.79} \\
\hline
50 subjects & 43.43 & 44.50 & 49.76 & \textbf{51.80} & 47.75 & 45.14 & 45.59 & 51.58 \\
\hline
\end{tabular}
\end{center}
\end{table*}

\begin{table*}[!t]
\renewcommand{\arraystretch}{1}
\caption{The clustering results of different algorithms on the FEI database. The best results are in bold font.}\label{Tab_FEI}
\begin{center}
\begin{tabular}{|c|C{1cm}|C{1cm}|C{1cm}|C{1cm}|C{1cm}|C{1cm}|C{1cm}|C{1cm}|}
\hline
\multirow{2}{*}{k}
  & \multicolumn{8}{c|}{Accuracy}\\
    \cline{2-9}
  & Kmeans & SSC & LRR & LSR & CASS & MoG & NSSC & Ours \\
\hline
\hline
5 subjects & 81.43 & 90.00 & 81.43 & 88.57 & 95.71 & 80.00 & 84.29 & \textbf{98.57}\\
\hline
10 subjects & 65.00 & 71.43 & 70.71 & 72.14 & 80.00 & 66.43 & 70.00 & \textbf{85.71}\\
\hline
15 subjects & 68.57 & 80.00 & 69.05 & 65.23 & 78.57 & 62.38 & 71.90 & \textbf{82.38}\\
\hline
20 subjects & 66.79 & 73.93 & 71.43 & 70.00 & \textbf{75.36} & 65.36 & 71.01 & 72.50\\
\hline
30 subjects & 64.52 & \textbf{76.19} & 59.29 & 65.48 & 67.86 & 66.67 & 66.43 & 69.29\\
\hline
40 subjects & 61.07 & \textbf{77.14} & 57.86 & 64.46 & 65.36 & 63.75 & 66.07 & 66.07\\
\hline
\multirow{2}{*}{k}
  & \multicolumn{8}{c|}{ Normalized Mutual Information}\\
    \cline{2-9}
  & Kmeans & SSC & LRR & LSR & CASS & MoG & NSSC & Ours \\
\hline
\hline
5 subjects & 80.51 & 82.33 & 77.65 & 81.95 & 93.24 & 69.24 & 76.57 & \textbf{96.77}\\
\hline
10 subjects & 70.03 & 70.21 & 76.47 & 74.20 & 77.58 & 63.70 & 73.02 & \textbf{89.44}\\
\hline
15 subjects & 79.85 & 79.40 & 77.90 & 69.72 & 83.74 & 66.26 & 78.59 & \textbf{85.85}\\
\hline
20 subjects & 79.49 & 77.74 & 81.34 & 74.72 & \textbf{81.93} & 71.80 & 75.14 & 80.64\\
\hline
30 subjects & 77.37 & \textbf{81.64} & 76.76 & 75.27 & 79.99 & 75.65 & 78.59 & 81.22\\
\hline
40 subjects & 78.29 & \textbf{83.48} & 76.54 & 76.08 & 0.7906 & 76.59 & 78.67 & 80.48\\
\hline
\end{tabular}
\end{center}
\end{table*}
For the parameter $c$, we can see that the comparison methods have no parameter $c$ and always give a straight line. Note that our method can give the best performance when $c$ is smaller than 1 on the Hopkins 155, USPS, C-Cube and Extended Yale B databases. Especially for Extended Yale B, the accuracy of our method is almost 100 percent. For FEI, the accuracy of our method is highest when $c=0.01$. Therefore, our method has the ability to achieve the best performance for the whole databases. Note that when the value of $c$ is larger than 0 or 1, the performance of our method tends to decrease rapidly. From our objective function (\ref{15}), we can see that when parameter $c$ increases, the noise term can be very small for all situations which directly reduces the ability of our objective function to suppress the large noise.
Hence, using Cauchy loss function to deal with the noise term is powerful to reduce the influence of the noise on subspace clustering. The best parameters of each method for the experiments on the whole databases are listed in Table \ref{Tab_para}.
\begin{table*}[!t]
\renewcommand{\arraystretch}{1}
\caption{The clustering results of different algorithms on the Extended Yale B database. The best results are in bold font.}\label{Tab_YaleB}
\begin{center}
\begin{tabular}{|c|C{1cm}|C{1cm}|C{1cm}|C{1cm}|C{1cm}|C{1cm}|C{1cm}|C{1cm}|}
\hline
\multirow{2}{*}{k}
  & \multicolumn{8}{c|}{Accuracy}\\
    \cline{2-9}
  & Kmeans & SSC & LRR & LSR & CASS & MoG & NSSC & Ours \\
\hline
\hline
5 subjects & 24.06 & 78.75 & 80.63 & 84.36 & 84.06 & 85.00 & 88.44 & \textbf{95.00}\\
\hline
8 subjects & 15.63 & 60.74 & 60.55 & 75.78 & 72.46 & \textbf{83.59} & 58.01 & \textbf{83.59}\\
\hline
10 subjects & 13.59 & 60.47 & 60.62 & 66.09 & 75.00 & 62.78 & 49.69 & \textbf{80.31}\\
\hline
\multirow{2}{*}{k}
  & \multicolumn{8}{c|}{ Normalized Mutual Information}\\
    \cline{2-9}
  & Kmeans & SSC & LRR & LSR & CASS & MoG & NSSC & Ours \\
\hline
\hline
5 subjects & 1.24 & 69.51 & 64.39 & 73.10 & 73.17 & 69.22 & 78.20 & \textbf{90.65}\\
\hline
8 subjects & 0.69 & 56.78 & 55.68 & 69.27 & 66.90 & 76.78 & 52.72 & \textbf{78.00}\\
\hline
10 subjects & 1.20 & 58.66 & 56.15 & 57.81 & 72.50 & 62.78 & 47.69 & \textbf{77.37}\\
\hline
\end{tabular}
\end{center}
\end{table*}

\begin{table*}[t]
\renewcommand{\arraystretch}{1}
\caption{Computation time of different algorithms on the FEI dataset as a function of the number of subjects.}\label{Tab_runningtime}
\begin{center}
\begin{tabular}{|c||c|c|c|c|c|c|c|c|}
\hline
k & Kmeans & SSC & LRR & LSR & CASS & MoG & NSSC & Ours \\
\hline
5 subjects & 0.03 & 38.27 & 1.09 & 0.04 & 2.27 & 11.38 & 0.13 & 0.32\\
\hline
10 subjects & 0.09 & 84.68 & 1.22 & 0.11 & 10.77 & 72.01 & 0.22 & 0.51 \\
\hline
15 subjects & 0.18 & 149.33 & 1.71 & 0.17 & 31.13 & 292.68 & 0.36 & 0.84 \\
\hline
20 subjects & 0.29 & 239.44 & 2.39 & 0.26 & 60.34 & 728.12 & 0.56 & 1.40 \\
\hline
30 subjects & 0.63 & 495.24 & 3.87 & 0.57 & 159.61 & 3247.61 & 1.13 & 2.45  \\
\hline
40 subjects & 1.10 & 893.98 & 6.08 & 1.29 & 321.43 & 12637.33 & 2.22 & 5.25 \\
\hline
\end{tabular}
\end{center}
\end{table*}

\subsection{Experimental results}

Table \ref{Tab_Hopkins}, \ref{Tab_USPS}, \ref{Tab_CCC}, \ref{Tab_FEI} and \ref{Tab_YaleB} give the experimental results of different methods on the Hopkins 155, USPS, C-Cube, FEI and Extended Yale B databases, respectively.
From Table \ref{Tab_Hopkins}, we can see that MoG and our method give the best performance on the average accuracy of the whole videos. But the corresponding NMI of MoG is lower than our proposed method. For the 3 motions situation, our method gives the best results both on the metrics AC and NMI. The medians of our method for 2 motions and total cases can reach 100 percent which shows the superiority of our proposed method. Although the accuracy of MoG is slightly bigger than our method for 2 motions, our method gives better quality clustering results through balancing all the cases.
For the USPS data set, our method outperforms other algorithms for the whole situations. Especially for the case of 5 subjects, the accuracy of our method is more than 7 percent better than the second best result. Note that the Kmeans algorithm gives the better performance than CASS and MoG on the USPS database, which means the handwritten digit data perhaps lacks the subspace structure. Even so, our method still shows its effectiveness on this data set.
For C-Cube, we can see that SSC shows good performance for 20 subjects based on AC, and LSR gives the highest NMI for 50 subjects. However, our method outperforms other methods in eight out of ten total cases. In particular, the AC value of our method is more than 4 percent higher than the second best result.
From Table \ref{Tab_FEI}, we can see that SSC outperforms other methods for 30 and 40 subjects, and CASS gives the best performance with 20 subjects. These subspace clustering results can be attributed to the subspace preserving of sparseness. For the remaining cases, our method can achieve the best clustering results. In particular, the accuracy of our method is nearly 99 percent for the 5 subjects.
Table \ref{Tab_YaleB} shows the clustering results on the Extended Yale B database. It shows that our method outperforms state-of-the-art methods for all these three clustering tasks, and MoG gives the same accuracy with our method for 8 subjects. Especially for the case of 5 subjects, the accuracy of our method is higher than the second best result by 10 percent which is a significant improvement. Note that Kmeans gives a very bad performance on the Extended Yale B database which means that the performance of Kmeans algorithm is easily influenced by the noise in the data. As stated in section \ref{ExperimentA}, the  Extended Yale B database contains the large noise. Therefore, this experiment can further verify the effectiveness of our method in handling the noise.

In summary, our proposed method is more robust to the noise and outperforms other state-of-the-art methods on the whole databases. It is sufficient to verify that our method is capable of finding the underlying subspace structure and clustering the data points into their subspaces.

\subsection{Computational Complexity Analysis}

As shown in Algorithm \ref{alg1}, the computation cost of our iterative algorithm depends on the computation of $\bf{Z}$, $Q$ and $\bf{R}$. The main computation cost of $\bf{Z}$ is the computation of ${\left( {{Q}^{t + 1}}{{{\bf{X}}^T}{\bf{X}} + 2\lambda {\bf{I}}} \right)^{ - 1}}$ which is $\mathbf{O}({n^3})$. For $Q$, its time cost is the computation of $\left\| {{{\bf{R}}^{t + 1}}} \right\|_F^2$ which is $\mathbf{O}({n^2})$. The computational cost for $\bf{R}$ is $\mathbf{O}(d{n^2})$. Therefore, the overall time complexity of our optimization method is $\mathbf{O}(t{n^3}+td{n^2})$, where $t$ denotes the number of iterations.

Furthermore, we give the computation time of different algorithms. Due to space limit, we only report the running time of all compared methods on the FEI data set which is shown in Table \ref{Tab_runningtime}. Note that the results are based on the codes implemented by their authors.  The calculations are performed using an Intel(R) Core(TM) i3-2130 CPU @ 3.40GHz with 16.00GB memory and 64-bit Windows7 operating system. It can be seen that the computation time of LSR is lower than other subspace clustering methods. This comes from the fact that LSR can directly obtain a closed-form solution without using an iterative way. However, SSC, CASS and MoG consume more time than other methods. Especially for MoG, its computation time increases drastically in the number of subjects. As for LRR, NSSC and our method, the computational cost of them is moderate for all situations.



\section{Conclusion \label{section 6}}
In this paper, we propose a robust subspace clustering method based on Cauchy loss function (CLF). To this end, we use CLF to penalize the noise term for suppressing the large noise mixed in the real data. Due to that the CLF's influence function has a upper bound, it can alleviate the influence of a single sample, especially the sample with a large noise, on estimating the residuals. Furthermore, we theoretically prove the grouping effect of our proposed method, and present its convergence analysis. Finally, experimental results on five real datasets reveal that our proposed method outperforms several representative methods.
\ifCLASSOPTIONcaptionsoff
  \newpage
\fi



%
{
\bibliographystyle{IEEEtran}
\bibliography{egbib}
}

\begin{IEEEbiographynophoto}{Xuelong Li}
(M'02-SM'07-F'12) is a full professor with the Xi'an Institute of Optics and Precision Mechanics, Chinese Academy of Sciences, Xi'an 710119, Shaanxi, P. R. China.
\end{IEEEbiographynophoto}
\ \\
\ \\
\ \\

\begin{IEEEbiography}[{\includegraphics[width=1in,height=1.32in,clip,keepaspectratio]{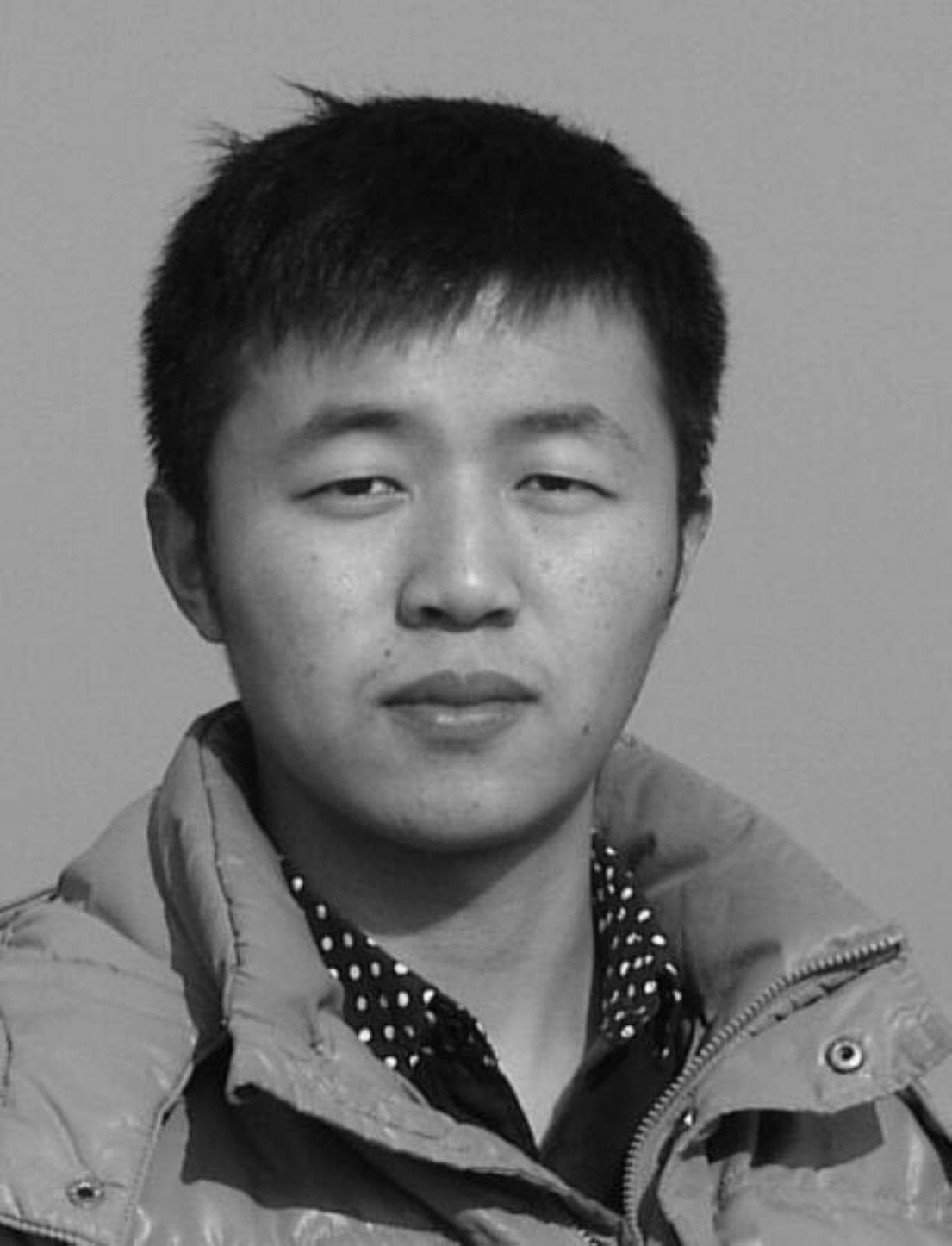}}]{Quanmao Lu}
is currently working toward the Ph.D. degree in the Center for Optical Imagery Analysis and Learning, Xi'an Institute of Optics and Precision Mechanics, Chinese Academy of Sciences, Xi'an, China.
His current research interests include machine learning and computer vision.
\end{IEEEbiography}

\ \\

\begin{IEEEbiography}[{\includegraphics[width=1in,height=1.32in,clip,keepaspectratio]{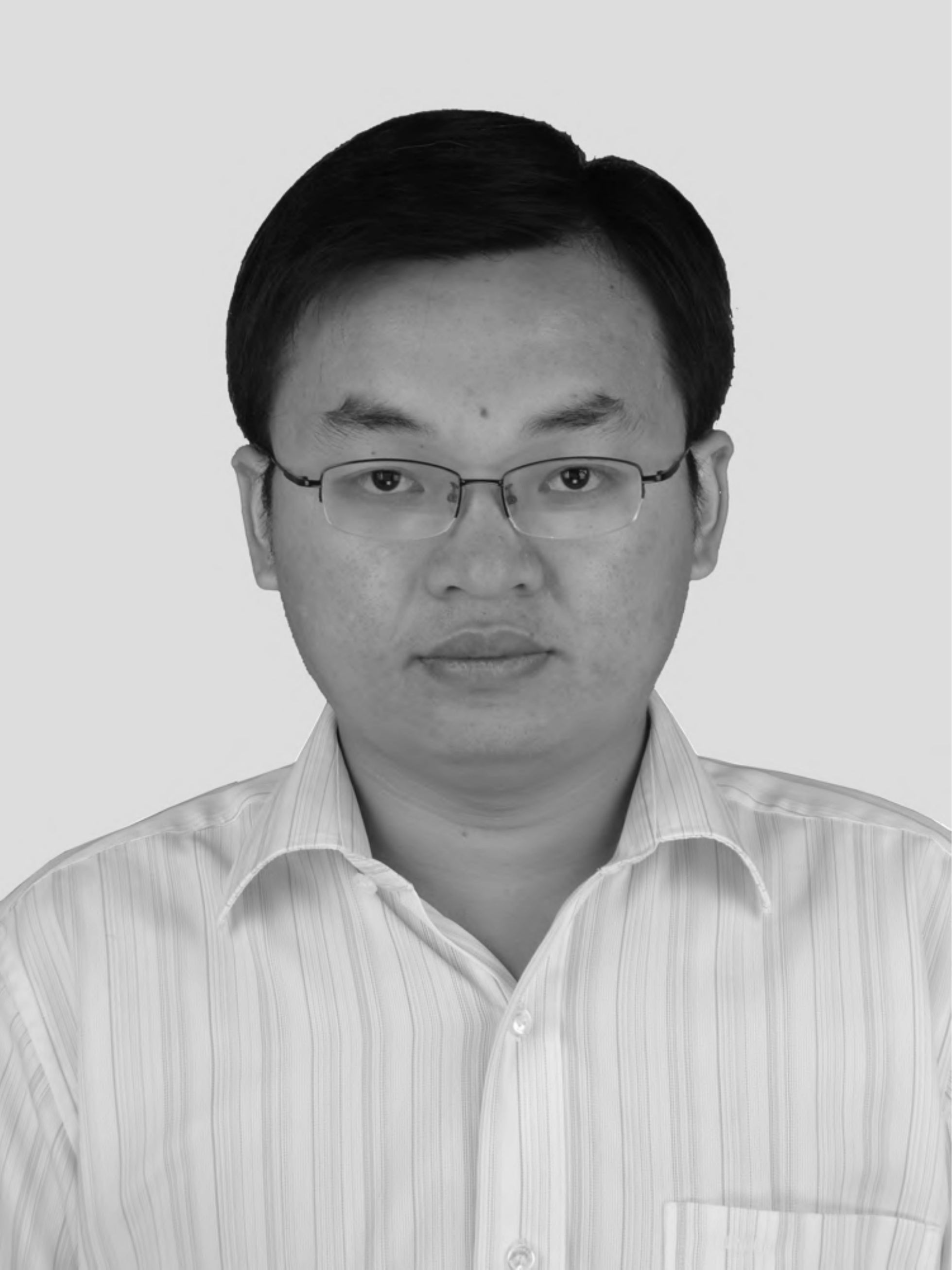}}]{Yongsheng Dong}%
(M'14) received his Ph. D. degree in applied mathematics from Peking University in 2012. He was a postdoctoral research fellow with the Center for Optical Imagery Analysis and Learning, Xi'an Institute of Optics and Precision Mechanics, Chinese Academy of Sciences, Xi'an, China from 2013 to 2016. From 2016 to 2017, he was a visiting research fellow at the School of Computer Science and Engineering, Nanyang Technological University, Singapore. He is currently an associate professor with the School of Information Engineering, Henan University of Science and Technology, China. His current research interests include pattern recognition, machine learning, and computer vision.

\end{IEEEbiography}

\begin{IEEEbiography}[{\includegraphics[width=1in,height=1.32in,clip,keepaspectratio]{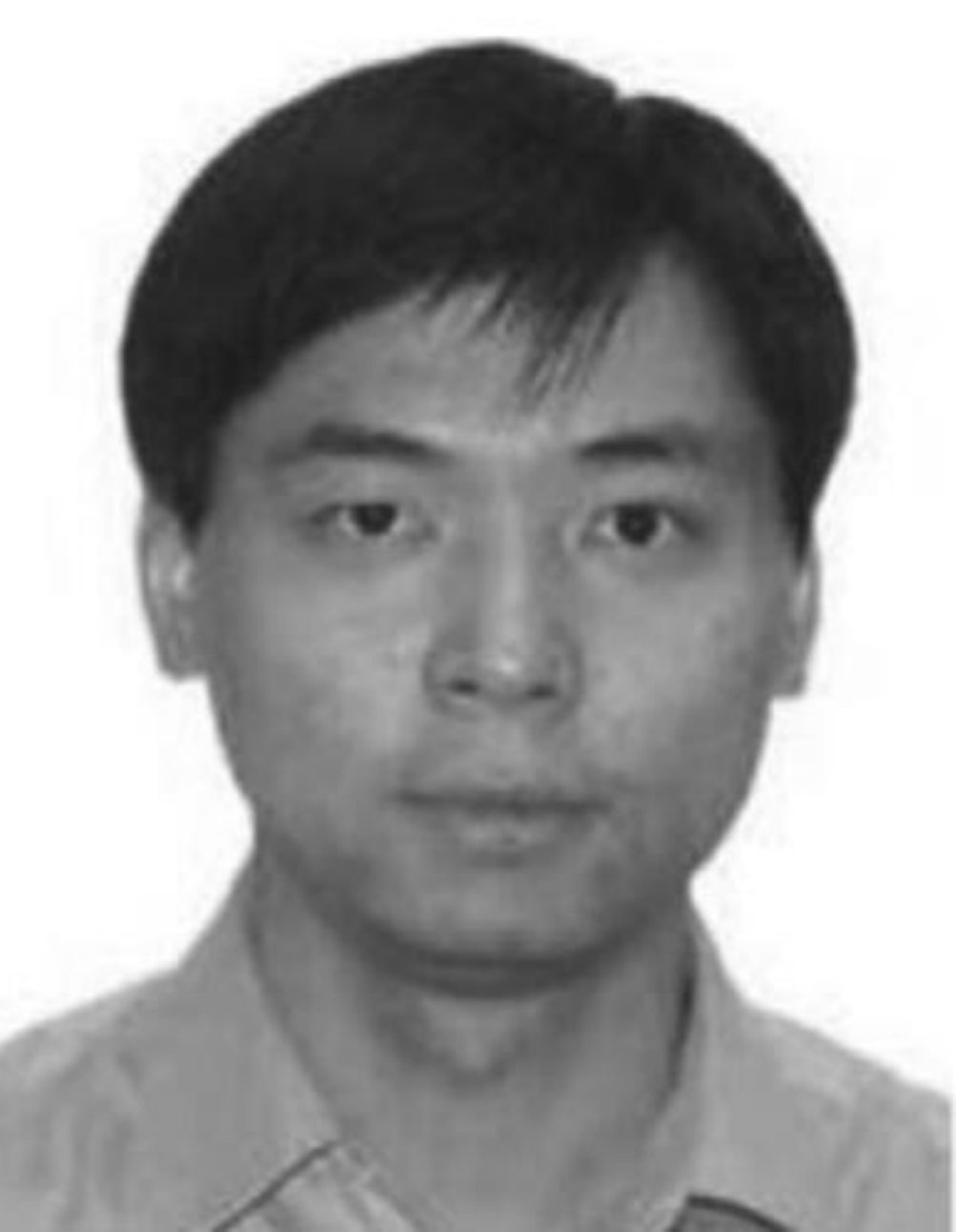}}]{Dacheng Tao}%
(F'15) was a Professor of Computer
Science and the Director of the Centre for Artificial
Intelligence, University of Technology Sydney,
Ultimo, NSW, Australia. He is currently a Professor
of Computer Science and an ARC Future
Fellow with the Faculty of Engineering and Information
Technologies, School of Information Technologies,
The University of Sydney, Sydney, NSW,
Australia, where he is also the Founding Director of
the UBTech Sydney Artificial Intelligence Institute.
He mainly applies statistics and mathematics to
artificial intelligence and data science. His current research interests include
computer vision, data science, image processing, machine learning, and video
surveillance. His research results have expounded in one monograph and
over 500 publications at prestigious journals and prominent conferences,
such as the IEEE TRANSACTIONS ON PATTERN ANALYSIS AND MACHINE
INTELLIGENCE, the IEEE TRANSACTIONS ON NEURAL NETWORKS AND
LEARNING SYSTEMS, the IEEE TRANSACTIONS ON IMAGE PROCESSING,
the Journal of Machine Learning Research, the International Journal of
Computer Vision, NIPS, CIKM, ICML, CVPR, ICCV, ECCV, AISTATS,
ICDM, and ACM SIGKDD.

Dr. Tao is a fellow of OSA, IAPR, and SPIE. He received several best paper
awards, such as the best theory/algorithm paper runner up award in the IEEE
ICDM07, the Best Student Paper Award in the IEEE ICDM13, and the 2014
ICDM 10-year highest-impact paper award. He received the 2015 Australian
Scopus-Eureka Prize, the 2015 ACS Gold Disruptor Award, and the 2015
UTS Vice-Chancellors Medal for Exceptional Research.

\end{IEEEbiography}

\ \\
\ \\
\ \\
\ \\
\ \\
\ \\
\ \\
\ \\
\ \\
\ \\
\ \\
\ \\
\ \\
\ \\
\ \\
\ \\
\ \\
\ \\
\ \\
\ \\
\ \\
\ \\
\ \\
\ \\
\ \\
\ \\
\ \\
\ \\
\ \\
\ \\
\ \\
\ \\
\ \\
\ \\
\ \\
\ \\
%

%
%
%




\end{document}